\documentclass[letterpaper,11pt]{article}
\usepackage[toc,page]{appendix}
\usepackage[margin=1in]{geometry}
\usepackage[bookmarks, colorlinks=true, plainpages = false, citecolor = blue,linkcolor=red,urlcolor = blue, filecolor = blue,pagebackref]{hyperref}
\usepackage{url}
\usepackage{amsmath,amsfonts,amsthm,amssymb,bm,verbatim,dsfont,mathtools}
\usepackage{color,graphicx,appendix}
\usepackage{etoolbox}
\usepackage{tikz}
\usepackage{xr,xspace}
\usepackage{todonotes}
\usepackage{paralist}
\usepackage{caption,subcaption,soul}
\usepackage[ruled,vlined,linesnumbered]{algorithm2e}
\usepackage{appendix}
\makeatletter

\newtheorem{theorem}{Theorem}
\newtheorem{lemma}{Lemma}
\newtheorem{proposition}{Proposition}
\newtheorem{corollary}{Corollary}
\theoremstyle{definition}

\newtheorem{remark}{Remark}

\usepackage{xspace,prettyref}

\newcommand{\diverge}{\to\infty}

\newcommand{\reals}{{\mathbb{R}}}

\newcommand{\naturals}{{\mathbb{N}}}

\newcommand{\diff}{{\rm d}}

\newcommand{\Expect}{\mathbb{E}}
\newcommand{\expect}[1]{\mathbb{E}\left[ #1 \right]}

\newcommand{\Prob}{\mathbb{P}}

\newcommand{\prob}[1]{ \mathbb{P}\left\{ #1 \right\} }

\newcommand{\toprob}{\xrightarrow{\Prob}}

\newcommand{\iid}{i.i.d.\xspace}
\newrefformat{eq}{(\ref{#1})}
\newrefformat{chap}{Chapter~\ref{#1}}
\newrefformat{sec}{Section~\ref{#1}}
\newrefformat{alg}{Algorithm~\ref{#1}}
\newrefformat{fig}{Fig.~\ref{#1}}
\newrefformat{tab}{Table~\ref{#1}}
\newrefformat{rmk}{Remark~\ref{#1}}
\newrefformat{clm}{Claim~\ref{#1}}
\newrefformat{def}{Definition~\ref{#1}}
\newrefformat{cor}{Corollary~\ref{#1}}
\newrefformat{lmm}{Lemma~\ref{#1}}
\newrefformat{prop}{Proposition~\ref{#1}}
\newrefformat{app}{Appendix~\ref{#1}}
\newrefformat{hyp}{Hypothesis~\ref{#1}}
\newrefformat{thm}{Theorem~\ref{#1}}
\newrefformat{ass}{Assumption~\ref{#1}}

\newcommand{\pth}[1]{\left( #1 \right)}
\newcommand{\qth}[1]{\left[ #1 \right]}
\newcommand{\sth}[1]{\left\{ #1 \right\}}

\newcommand{\abth}[1]{\left | #1 \right |}

\newcommand{\norm}[1]{\left\|{#1} \right\|}

\newcommand{\lnorm}[2]{\left\|{#1} \right\|_{{#2}}}

\newcommand{\Fnorm}[1]{\lnorm{#1}{\rm F}}

\newcommand{\iprod}[2]{\left \langle #1, #2 \right\rangle}
\newcommand{\Iprod}[2]{\langle #1, #2 \rangle}
\newcommand{\indc}[1]{{\mathbf{1}_{\left\{{#1}\right\}}}}

\newcommand{\calA}{{\mathcal{A}}}
\newcommand{\calB}{{\mathcal{B}}}

\newcommand{\calE}{{\mathcal{E}}}
\newcommand{\calF}{{\mathcal{F}}}

\newcommand{\calH}{{\mathcal{H}}}
\newcommand{\calI}{{\mathcal{I}}}

\newcommand{\calK}{{\mathcal{K}}}

\newcommand{\calN}{{\mathcal{N}}}

\newcommand{\calS}{{\mathcal{S}}}

\newcommand{\im}{\mathsf{i}}
\newcommand{\poly}{\mathsf{poly}}

\renewcommand{\hat}{\widehat}
\renewcommand{\tilde}{\widetilde}

\begin{document}

\title{On Learning Over-parameterized Neural Networks: 
A Functional Approximation Perspective}

%
\author{
Lili Su \\
CSAIL, MIT\\
{lilisu@mit.edu}
\and 
Pengkun Yang\\
Department of Electrical Engineering\\
Princeton University\\
{pengkuny@princeton.edu}
}

\date{\today}

\maketitle

\begin{abstract}
We consider training over-parameterized two-layer neural networks with Rectified Linear Unit (ReLU) using gradient descent (GD) method. 
Inspired by a recent line of work, we study the evolutions of network prediction errors across GD iterations, which can be  neatly described in a matrix form. 
When the network is sufficiently over-parameterized, these matrices individually approximate {\em an} integral operator which is determined by 
the feature vector distribution $\rho$ only. 
Consequently, GD method can be viewed as {\em approximately} applying the powers of this integral operator on the underlying/target function $f^*$ that generates the responses/labels. 
 
We show that if $f^*$ admits a low-rank approximation with respect to the eigenspaces of this integral operator, then
the empirical risk decreases to this low-rank approximation error at a linear rate 
which is determined by $f^*$ and $\rho$ only, i.e., the rate is independent of the sample size $n$. 
Furthermore, if $f^*$ has zero low-rank approximation error, 
then, as long as the width of the neural network is $\Omega(n\log n)$, 
the empirical risk decreases to $\Theta(1/\sqrt{n})$. 
To the best of our knowledge, this is the first result showing the sufficiency of nearly-linear network over-parameterization.  We provide an application of our general results to the setting where $\rho$ is the uniform distribution on the spheres and $f^*$ is a polynomial. 
Throughout this paper, we consider the scenario where the input dimension $d$ is fixed. 

\end{abstract}

\section{Introduction}
\label{sec: intro}
Neural networks have been successfully applied in many real-world machine learning applications. However, a thorough understanding of the theory behind their practical success, even for two-layer neural networks, is still lacking. For example, despite learning optimal neural networks is provably NP-complete \cite{brutzkus2017globally,blum1989training}, 
in practice, even the neural networks found by the simple first-order methods perform well \cite{krizhevsky2012imagenet}. Additionally, in sharp contrast to traditional learning theory, over-parameterized neural networks (more parameters than the size of the training dataset) are observed to enjoy smaller training and even smaller generalization errors \cite{zhang2016understanding}. In this paper, we focus on training over-parameterized two-layer neural networks with Rectified Linear Unit (ReLU) using gradient descent (GD) method. Our results can be extended to other activation functions that satisfy some regularity conditions; see \cite[Theorem 2]{ghorbani2019linearized} for an example.  
The techniques derived and insights obtained in this paper might be applied to deep neural networks as well, for which similar matrix representation exists \cite{du2018gradient}. 

Significant progress has been made in understanding the role of over-parameterization in training neural networks with first-order methods \cite{allen2018learning,du2018gradient,arora2019fine,oymak2019towards,mei2018mean,li2018learning,zou2018stochastic,du2018gradientlee,allen2018convergence,cao2019generalization}; with proper random network initialization, (stochastic) GD converges to a (nearly) global minimum provided that the width of the network $m$ is {\em polynomially} large in the size of the training dataset $n$. However, 
neural networks seem to interpolate the training data as soon as the number of parameters exceed the size of the training dataset by a constant factor \cite{zhang2016understanding,oymak2019towards}.
%
To the best of our knowledge, a provable justification of why such mild over-parametrization is sufficient for successful gradient-based training is still lacking. 
Moreover, the convergence rates derived in many existing work approach 0 as $n\diverge$; see Section \ref{sec: related work} for details. In many applications the volumes of the datasets are huge --  the ImageNet 
dataset \cite{Imagenet} has 14 million images. For those applications, a non-diminishing (i.e., constant w.\,r.\,t.\,$n$) convergence rate is more desirable.
 In this paper, our goal is to characterize a {\em constant} (w.\,r.\,t.\,$n$) convergence rate while improving the sufficiency guarantee of network over-parameterization. Throughout this paper, we focus on 
 the setting where the dimension of the feature vector $d$ is fixed, leaving the high dimensional region as one future direction. 

Inspired by a recent line of work 
\cite{du2018gradient,arora2019fine}, 
we focus on characterizing the evolutions of the neural network prediction errors under GD method. 
 This focus is motivated by the fact that the neural network representation/approximation of a given function might not be unique \cite{klusowski2018approximation}, and this focus is also validated by experimental neuroscience \cite{marder2006variability,athalye2018evidence}. 



\paragraph{Contributions}
It turns out that the evolution of the network prediction error can be neatly described in a matrix form.  
When the network is sufficiently over-parameterized, the matrices involved 
individually approximate {\em an} integral operator which is determined by 
the feature vector distribution $\rho$ only. 
Consequently, GD method can be viewed as {\em approximately} applying the powers of this integral operator on the underlying/target function $f^*$ that generates the responses/labels. The advantages of taking such a functional approximation perspective are three-fold: 
\begin{itemize}
	\item We showed in Theorem \ref{thm: vacuous rate} and Corollary \ref{col: corollary of diminising eigenvalues} that the existing rate characterizations in the influential line of work \cite{du2018gradient,arora2019fine,du2018gradientlee} approach zero (i.e., $\to 0$) as 
	 $n\to \infty$. This is because the spectra of these matrices, as $n$ diverges, concentrate on the spectrum of the integral operator, in which the unique limit of the eigenvalues is zero. 

	\item 
We show in Theorem \ref{thm: main theorem} 
that the training convergence rate is determined by how $f^*$ can be decomposed into the eigenspaces of an integral operator. 
This observation is also validated by a couple of empirical observations: (1) The spectrum of the MNIST data concentrates on the first a few eigenspaces \cite{lecun1998gradient}; and (2) the training is slowed down if labels are partially corrupted \cite{zhang2016understanding,arora2019fine}.  
	\item We show in Corollary \ref{thm: main theorem zero approximation error} that if $f^*$ can be decomposed into a finite number of eigenspaces of the integral operator, then $m=\Theta(n\log n)$ is sufficient for the training error to converge to $\Theta(1/\sqrt{n})$ with a constant convergence rate. 
 To the best of our knowledge, this is the first result showing the sufficiency of nearly-linear network over-parameterization. 
\end{itemize}

\paragraph{Notations}
For any $n, m\in \naturals$, let $\qth{n} := \sth{1, \cdots, n}$ and $\qth{m} := \sth{1, \cdots, m}$. 
For any $d\in \naturals$, denote the unit sphere as $\calS^{d-1} : = \sth{x: ~ x\in \reals^d, ~ \& ~ \norm{x} =1}$, where $\norm{\cdot}$ is the standard $\ell_2$ norm when it is applied to a vector. We also use $\norm{\cdot}$ for the spectral norm when it is applied to a matrix. The Frobenius norm of a matrix is denoted by $\norm{\cdot}_{F}$. 
Let $L^2(\calS^{d-1}, \rho) $ denote the space of functions with finite norm, where the inner product $\iprod{\cdot}{\cdot}_{\rho}$ and $\|\cdot\|_{\rho}^2$ are defined as $\iprod{f}{g}_{\rho} := \int_{\calS^{d-1}}f(x)g(x) \diff \rho(x)$ and $\|f\|_{\rho}^2 := \int_{\calS^{d-1}}f^2(x) \diff \rho (x) <\infty$. 
We use standard Big-$O$ notations, 
e.g., for any sequences $\sth{a_r}$ and $\sth{b_r}$, we say $a_r = O(b_r)$ or
$a_r \lesssim b_r$ if there is an absolute constant $c>0$ such that $\frac{a_r}{b_r} \le c$, 
we say $a_r = \Omega(b_r)$ or
$a_r \gtrsim b_r$ if $b_r=O(a_r)$
and we say $a_r=\omega(b_r)$ if $\lim_{r\diverge} |a_r/b_r| = \infty$.

\section{Related Work}
\label{sec: related work}
The volume of the literature on neural networks is growing rapidly, and we cannot hope to do justice to this large body of related work. 
Here we sample an incomplete list of works that are most relevant to this work. 

There have been intensive efforts in proving the (global) convergence of the simple first-order methods such as (stochastic) gradient descent \cite{brutzkus2017globally,li2017convergence,zhong2017recovery}, where the true function that generates the responses/labels is a two-layer neural network of the same size as the neural network candidates.
Notably, in this line of work, it is typically assumed that $m\le d$. 
 
Over-parameterized neural networks are observed to enjoy smaller training errors and even smaller generalization errors \cite{zhang2016understanding,li2018learning}. Allen-Zhu et al.\,\cite{allen2018learning} considered the setting where the true network is much smaller than the candidate networks, and showed that 
searching among over-parametrized network candidates smoothes the optimization trajectory and enjoys a strongly-convex-like behavior. 
Similar results 
were shown in \cite{du2018gradient,arora2019fine, oymak2019towards}. 
In particular, it was shown in an inspiring work \cite{du2018gradient}
that when $m = \Omega(n^6)$ 
and the minimum eigenvalue of some Gram matrix is positive, 
then randomly initialized GD can find an optimal neural network, under squared loss, at a linear rate. 
 However, the involved minimum eigenvalue scales in $n$, and the impacts of this scaling on the convergence and the corresponding convergence rate 
 were overlooked in \cite{du2018gradient}.  Unfortunately, taking such scaling into account, their convergence rate approaches 0 as $n\diverge$; we formally show this in Theorem \ref{thm: vacuous rate} and Corollary \ref{col: corollary of diminising eigenvalues}.   
Recently, 
\cite{oymak2019towards} showed that when $m=\Omega(n^2)$, the empirical risk (training error) goes to zero at a linear rate 
$(1-c\frac{d}{n})^t = \exp \pth{-t \log (1/(1-c\frac{d}{n}))}$, where $t$ is the GD iteration 
and $c>0$ is some absolute constant; see \cite[Corollaries 2.2 and 2.4]{oymak2019towards} for details. Here $\log (1/(1-c\frac{d}{n}))$ is the convergence rate. 
It is easy to see that 
$\log (1/(1-c\frac{d}{n})) \to 0$ as $n$ increases. 


For deep networks (which contain more than one hidden layer), the (global) convergence of (S)GD 
are shown in \cite{zou2018stochastic,du2018gradientlee,allen2018convergence} with different focuses and characterizations of over-parameterization sufficiency. In particular, \cite{zou2018stochastic} studied the binary classification problem and showed that (S)GD  can find a global minimum provided that the feature vectors with different labels are well separated and $m=\Omega(\poly(n, L))$, where $L$ is the number of hidden layers. Allen-Zhu et al.\, \cite{allen2018convergence} considered the regression problem and showed similar over-parameterization sufficiency. 
The over-parameterization sufficiency in terms of its scaling in $n$ is significantly improved in \cite{du2018gradientlee} without considering the scaling of the minimum eigenvalue of the Gram matrix in $n$.

All the above recent progress 
is established on the common observation that when the network is sufficiently over-parameterized, during training, the network weights are mainly stay within a small perturbation region centering around the initial weights. In fact, the over-parameterization sufficiency that ensures the above mild weight changes is often referred to as NTK region; see \cite{jacot2018neural} for details.  Recently, a handful of work studied linearized neural networks in high dimension \cite{ghorbani2019linearized,yehudai2019power,vempala2018gradient}. Since we consider fixed $d$, our results are not directly comparable to that line of work. 



\section{Problem Setup and Preliminaries}
%
\paragraph{Statistical learning}
We are given a training dataset $\sth{(x_i, y_i): 1\le i \le n}$ 
which consists of $n$ tuples $(x_i, y_i)$, where $x_i$'s are feature vectors that are 
identically and independently generated 
from a common but {\em unknown} distribution $\rho$ on $\reals^d$, and $y_i = f^*(x_i)$. 
We consider the problem of learning the unknown function $f^*$ with respect to the square loss. 
We refer to $f^*$ as a {\em target function}.
For simplicity, we assume 
$x_i\in \calS^{d-1}$ and $y_i\in [-1, 1]$. 
In this paper, we restrict ourselves to the family of $\rho$ that is absolutely continuous with respect to Lebesgue measure.
%
We are interested in finding a neural network to approximate $f^*$. In particular, we focus on  
two-layer fully-connected neural networks with ReLU activation,
i.e., 
\begin{align}
\label{eq: two layer model}
f_{\bm{W},\bm{a}} (x) = \frac{1}{\sqrt{m}}\sum_{j=1}^m a_j\qth{\iprod{x}{w_j}}_+, ~~ ~ \forall ~ x\in \calS^{d-1},
\end{align}
where 
$m$ is the number of hidden neurons and is assumed to be even, $\bm{W} =\pth{w_1, \cdots, w_m} \in \reals^{d\times m}$ are the weight vectors in the first layer, $\bm{a} = \pth{a_1, \cdots, a_m}$ with $a_j\in\sth{-1, 1}$ are the weights in the second layer, 
and $\qth{\cdot}_+ : = \max\sth{\cdot, 0}$ is the ReLU activation function.

Many authors assume $f^*$ is also a neural network \cite{mei2018mean,allen2018learning,saad1996dynamics,li2017convergence,tian2016symmetry}. 
Despite this popularity, a target function $f^*$ is not necessarily a neural network.
One advantage of working with $f^*$ directly is, as can be seen later, certain properties of $f^*$ are closely related to whether $f^*$ can be learned quickly by GD method or not. 
Throughout this paper, for simplicity, we do not consider the scaling in $d$ and treat $d$ as a constant.  
%
%

%
\paragraph{Empirical risk minimization via gradient descent} 
For each $k=1, \cdots, m/2$:  Initialize $w_{2k-1} \sim \calN(\bm{0},  {\bf I})$, and $a_{2k-1}=1$ with probability $\frac{1}{2}$, and  $a_{2k-1}=-1$ with probability $\frac{1}{2}$. Initialize $w_{2k} =w_{2k-1}$ and $a_{2k}= -a_{2k-1}$.  
All randomnesses in this initialization are independent, and are independent of the dataset. 
This initialization is chosen to guarantee zero output at initialization.  Similar initialization is adopted in \cite[Section 3]{chizat2018note} and \cite{woodworth2019kernel}. \footnote{Our analysis might be adapted to other initialization schemes, such as He initialization, with $m=\Omega(n^2)$. 
Nevertheless, the more stringent requirement on $m$ might only be an artifact of our analysis. }
%
We fix the second layer $\bm{a}$ and optimize the first layer $\bm{W}$ through GD 
on the empirical risk w.\,r.\,t.\,square loss \footnote{The simplification assumption that the second layer is fixed is also adopted in  
\cite{du2018gradient,arora2019fine}. Similar frozen assumption is adopted in \cite{zou2018stochastic,allen2018convergence}. We do agree this assumption might restrict the applicability of our results. Nevertheless, even this setting is not well-understood despite the recent intensive efforts. 
}:   
\begin{align}
\label{eq: general function, iterates}
L_n(\bm{W}): = \frac{1}{2n} \sum_{i=1}^n \qth{\pth{y_i - f_{\bm{W}}(x_i)}^2}.  
\end{align}
For notational convenience, 
we drop the subscript $\bm{a}$ in $f_{\bm{W}, \bm{a}}$. 
The weight matrix $\bm{W}$ is update as
\begin{align}
\label{eq: GD update}
\bm{W}^{t+1} = \bm{W}^{t} - \eta \frac{\partial L_n(\bm{W}^t)}{\partial \bm{W}^t}, 
\end{align}
where $\eta>0$ is stepsize/learning rate, and $\bm{W}^{t}$ is the weight matrix at the end of iteration $t$ with $\bm{W}^{0}$ denoting the initial weight matrix.  
For ease of exposition, let 
\begin{align}
\label{eq: prediction at time t}
\hat{y}_i(t) ~ : = ~ f_{\bm{W}^t}(x_i) = \frac{1}{\sqrt{m}} \sum_{j=1}^m a_j\qth{\iprod{w_{j}^t}{x_i}}_+, ~~~ \forall ~ i=1, \cdots, n.
\end{align}
 Notably, $\hat{y}_i(0) =0$ for $i=1, \cdots, n$. 
It can be easily deduced from \eqref{eq: GD update} that $w_j$ is updated as 
\begin{align}
w_j^{t+1} & = w_j^{t} + \frac{\eta a_j}{n\sqrt{m}} \sum_{i=1}^n \pth{y_i - \hat{y}_i(t)}
  x_i \indc{\iprod{ w_j^{t}}{x_i}  >0} \label{eq: weight update new}. 
\end{align}

\paragraph{Matrix representation}
Let $\bm{y}\in \reals^n$ be the vector that stacks the responses of 
$\{(x_i, y_i)\}_{i=1}^n$. Let $\bm{\hat{y}}(t)$ be the vector that stacks $\hat{y}_i(t)$ for $i=1, \cdots, n$ at iteration $t$. Additionally, let $\calA : = \sth{j: ~ a_j=1}$ and $\calB : = \sth{j: ~ a_j=-1}.$
The evolution of $\pth{\bm{y} - \bm{\hat{y}}(t)}$ can be neatly described in a matrix form. 
Define matrices $\bm{H}^+, \tilde{\bm{H}}^+$, and $\bm{H}^-, \tilde{\bm{H}}^-$ in $ \reals^n\times \reals^n$ as: For $t\ge 0$, and $i, i^{\prime} \in [n]$, 
\begin{align}
\label{eq: symmetric kernel at time t positive subnet}
\bm{H}^+_{ii^{\prime}}(t+1) &=  \frac{1}{nm}\iprod{x_i}{x_{i^{\prime}}} \sum_{j \in \calA}   \indc{\iprod{w_{j}^{t}}{x_{i^{\prime}}}  >0}\indc{\iprod{w_{j}^{t}}{x_{i}}  >0}, \\
\label{eq: asymmetric kernel at time t positive subnet}
\tilde{\bm{H}}^+_{ii^{\prime}}(t+1) &=  \frac{1}{nm}\iprod{x_i}{x_{i^{\prime}}} \sum_{j \in \calA}   \indc{\iprod{w_{j}^{t}}{x_{i^{\prime}}}  >0}\indc{\iprod{w_{j}^{t+1}}{x_{i}} >0}, 
\end{align}
and $\bm{H}^-_{ii^{\prime}}(t+1) $, $\tilde{\bm{H}}^-_{ii^{\prime}}(t+1)$ are defined similarly by replacing the summation over all the hidden neurons in $\calA$ in \eqref{eq: symmetric kernel at time t positive subnet} and \eqref{eq: asymmetric kernel at time t positive subnet} by the summation over $\calB$. 
It is easy to see that both $\bm{H}^+$ and $\bm{H}^-$ are positive semi-definite. 
The only difference between $\bm{H}^+_{ii^{\prime}}(t+1)$ (or $\bm{H}^-_{ii^{\prime}}(t+1)$) and $\tilde{\bm{H}}^+_{ii^{\prime}}(t+1)$ (or $\tilde{\bm{H}}^-_{ii^{\prime}}(t+1)$) is that $\indc{\iprod{w_{j}^{t}}{x_{i}} >0}$ is used in the former, whereas $\indc{\iprod{w_{j}^{t+1}}{x_{i}}  >0}$ is adopted in the latter.  
When a neural network is sufficiently over-parameterized (in particular, $m=\Omega(\poly (n))$), 
 the sign changes of the hidden neurons are sparse; see \cite[Lemma 5.4]{allen2018learning} and \cite[Lemma C.2]{arora2019fine} for details. The sparsity in sign changes suggests that both $\tilde{\bm{H}}^+(t)\approx \bm{H}^+(t)$ and $\tilde{\bm{H}}^-(t)\approx \bm{H}^-(t)$ are approximately PSD.

\begin{theorem}
\label{thm: matrix representation}
For any iteration $t \ge 0$ and any stepsize $\eta>0$, it is true that 
\begin{align*}
& \pth{\bm{I}-  \eta \pth{\tilde{\bm{H}}^+(t+1) + \bm{H}^-(t+1)}} \pth{\bm{y} - \bm{\hat{y}}(t)} \\
& \qquad \qquad \qquad \qquad \qquad \qquad \qquad \le \pth{\bm{y} - \bm{\hat{y}}(t+1)} \\
& \qquad \qquad \qquad \qquad\qquad \qquad \qquad \quad \le \pth{\bm{I}- \eta\pth{\bm{H}^+(t+1) + \tilde{\bm{H}}^-(t+1)}} \pth{\bm{y} - \bm{\hat{y}}(t)},
\end{align*}
where the inequalities are entry-wise.
\end{theorem}
Theorem \ref{thm: matrix representation} says that when the sign changes are sparse, the dynamics of $ \pth{\bm{y} - \bm{\hat{y}}(t)}$ are governed by a sequence of 
PSD matrices. 
%
Similar observation is made in \cite{du2018gradient,arora2019fine}. 

\section{Main Results}
\label{sec: main results}
%
We first show (in Section \ref{subsec:asymptotic vacuous rates}) that the existing convergence rates that are derived based on minimum eigenvalues approach 0 as the sample size $n$ grows. 
Then, towards a non-diminishing convergence rate, we characterize (in  Section \ref{subsec: constant rate}) how the target function $f^*$ affects the convergence rate.

\subsection{Convergence rates based on minimum eigenvalues}
\label{subsec:asymptotic vacuous rates}
Let $\bm{H}: = \bm{H}^+(1) + \bm{H}^-(1)$.  
It has been shown in \cite{du2018gradient} that when the neural networks are sufficiently over-parameterized $m=\Omega(n^6)$,  
the convergence of $\norm{\bm{y} - \hat{\bm{y}}(t)}$ and the associated convergence rates with high probability can be upper bounded as \footnote{
Though a refined analysis of that in \cite{du2018gradient} is given by \cite[Theorem 4.1]{arora2019fine},  the analysis crucially relies on the convergence rate in \eqref{eq: proposed convergence rate}.
} 
\begin{align}
\label{eq: proposed convergence rate}
\nonumber
\norm{\bm{y} - \hat{\bm{y}}(t)} &\le \pth{1- \eta \lambda_{\min}(\bm{H}) }^{t}\norm{\bm{y}-\hat{\bm{y}}(0)}\\
& = \exp\pth{- t \log \frac{1}{1- \eta \lambda_{\min}(\bm{H})} }\norm{\bm{y}}, 
\end{align}
where $\lambda_{\min}(\bm{H})$ is the smallest eigenvalue of $\bm{H}$. Equality \eqref{eq: proposed convergence rate} holds because of $\hat{\bm{y}}(0) = \bm{0}$. 
In this paper, we refer to $\log \frac{1}{1- \eta \lambda_{\min}(\bm{H})} $ as {\em convergence rate}. 
The convergence rate here is quite appealing at first glance as 
it is {\em independent} of the target function $f^*$.  Essentially \eqref{eq: proposed convergence rate} says that no matter how the training data is generated, via GD, we can always find an over-parameterized neural network that  perfectly fits/memorizes all the training data tuples 
exponentially fast! 
Though the spectrum of the random matrix $\bm{H}$ can be proved to concentrate as $n$ grows, we observe that $\lambda_{\min}(\bm{H})$ converges to 0 as $n$ diverges, formally shown in Theorem \ref{thm: vacuous rate}. 
\begin{theorem}
\label{thm: vacuous rate}
For any data distribution $\rho$, 
there exists a sequence of non-negative real numbers $\lambda_1\ge \lambda_2\ge \dots$ (independent of $n$) satisfying $\lim_{i\diverge} \lambda_{i} =0$ 
such that, with probability $1-\delta$, 
\begin{equation}
\label{eq:eigen-concentration}
\sup_i|\lambda_i-\tilde\lambda_i|\le \sqrt{\frac{\log (4n^2/\delta)}{m}}+\sqrt{\frac{8 \log (4/\delta)}{n}}. 
\end{equation}
where $\tilde\lambda_1\ge \dots \ge \tilde\lambda_n$ are the spectrum of $\bm{H}$. 
In addition, if $m=\omega(\log n)$, we have 
\begin{equation}
\label{eq:vanish-lambda-min}
\lambda_{\min}(\bm{H})  \toprob  0,  ~~~~ \text{as } n\diverge,  
\end{equation}
where $\toprob$ denotes convergence in probability. 
\end{theorem}
A numerical illustration of the decay of $\lambda_{\min}(\bm{H})$ in $n$ can be found in Fig.\,\ref{fig:eigen-min}.
%
Theorem \ref{thm: vacuous rate} is proved in Appendix \ref{app: thm: vacuous rate}. 
As a consequence of 
Theorem \ref{thm: vacuous rate}, the convergence rate in \eqref{eq: proposed convergence rate} approaches zero as $n\diverge$. 
%
 \begin{corollary}
 \label{col: corollary of diminising eigenvalues}
 For any $\eta = O(1)$, it is true that $\log \frac{1}{1-\eta \lambda_{\min}(\bm{H})} \to 0$ as $n\diverge$. 
 %
 \end{corollary}
In Corollary \ref{col: corollary of diminising eigenvalues}, we restrict our attention to $\eta =O(1)$. 
This is because the general analysis of GD \cite{nesterov2018lectures} adopted by \cite{arora2019fine,du2018gradient}
requires that $(1-\eta \lambda_{\max}(\bm{H}))>0$, and 
by the spectrum concentration given in Theorem \ref{thm: vacuous rate}, the largest eigenvalue of $\bm{H}$ concentrates on some strictly positive value as $n$ diverges, i.e., $\lambda_{\max}(\bm{H}) = \Theta(1)$. 
Thus, if 
$\eta = \omega(1)$,  then $(1-\eta \lambda_{\max}(\bm{H}))< 0$ for any sufficiently large $n$, violating the condition assumed in  \cite{arora2019fine,du2018gradient}. 

Theorem \ref{thm: vacuous rate} essentially follows from two observations. Let $\bm{K} = \expect{\bm{H}}$, where the expectation is taken with respect to the randomness in the network initialization. It is easy to see that by standard concentration argument, for a given dataset, the spectrum of $\bm{K}$ and $\bm{H}$ are close with high probability. 
In addition,  the spectrum of $\bm{K}$, as $n$ increases, concentrates on the spectrum of the following integral operator $L_{\calK}$ on $L^2(\calS^{d-1}, \rho)$,
\begin{align}
\label{eq: int op}
(L_{\calK} f)(x) : = \int_{\calS^{d-1}} \calK(x, s) f(s) \diff \rho, 
\end{align}
with the kernel function: 
\begin{align}
\label{eq: initial kernel random}
\calK(x,s) : = \frac{ \iprod{x}{s}}{2\pi} \pth{\pi -\arccos \iprod{x}{s}} ~~~ \forall ~ x, s \in \calS^{d-1}, 
\end{align}
which is bounded over $\calS^{d-1}\times \calS^{d-1}$. 
In fact, $\lambda_1 \ge \lambda_2 \ge \cdots$ in Theorem \ref{thm: vacuous rate} are the eigenvalues of $L_{\calK}$. 
As $\sup_{x, s\in \calS^{d-1}} \calK(x, s) \le \frac{1}{2}$, it is true that $\lambda_i \le 1$ for all $i\ge 1$. 
Notably, by definition, $\bm{K}_{ii^{\prime}}=\expect{\bm{H}_{ii^{\prime}}} = \frac{1}{n}\calK(x_i,x_{i^{\prime}})$ is the 
empirical kernel matrix  on the feature vectors of the given dataset $\sth{(x_i, y_i): 1\le i \le n}$. 
A numerical illustration of the spectrum concentration of $\bm{K}$ is given in 
Fig.\,\ref{fig:eigen-concentrate}; see, also, \cite{xie2017diverse}.

Though a generalization bound is given in \cite[Theorem 5.1 and Corollary 5.2]{arora2019fine}, it is unclear how this bound scales in $n$. 
%
In fact, if we do not care about the structure of the target function $f^*$ and allow $\frac{\bm{y}}{\sqrt{n}}$ to be arbitrary, this generalization bound  might not decrease to zero as $n\diverge$. A detailed argument and a numerical illustration can be found in Appendix \ref{app: generalization bound}. 
\begin{figure}
     \centering
     \begin{subfigure}[t]{0.49\textwidth}
         \centering
         \includegraphics[width=\textwidth]{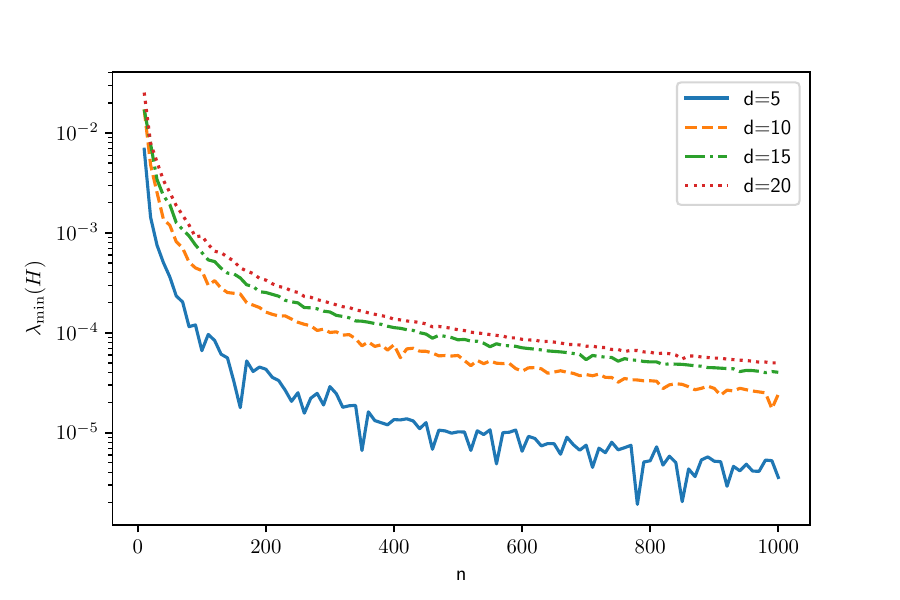}
          \caption{The minimum eigenvalues of one realization of $\bm{H}$ under different $n$ and $d$, with network width $m=2n$. }
           \label{fig:eigen-min}
          \end{subfigure}     
     \hfill    
     \begin{subfigure}[t]{0.49\textwidth}
         \centering
         \includegraphics[width=\textwidth]{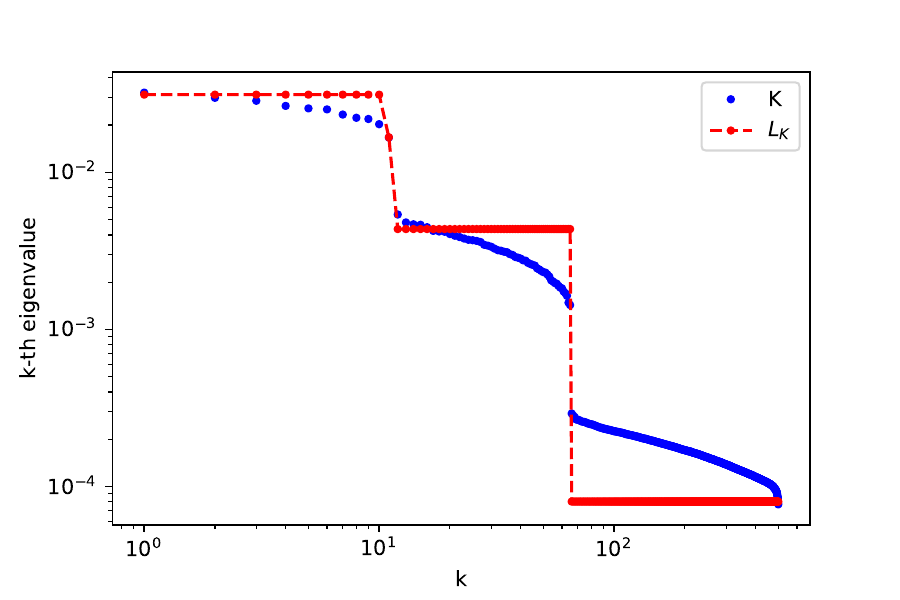}  
           \caption{The spectrum of $\bm{K}$ with $d=10,n=500$ concentrates around that of $L_{\calK}$. }
                 \label{fig:eigen-concentrate}
     \end{subfigure}
     \caption{The spectra of $\bm{H}$, $\bm{K}$, and $L_{\calK}$ when $\rho$ is the uniform distribution over $\calS^{d-1}$.}
\end{figure}

\subsection{Constant convergence rates}
\label{subsec: constant rate}
%
%

Recall that $f^*$ denotes the underlying function that generates output labels/responses (i.e., $y$'s) given input features (i.e., $x$'s).
%
For example, $f^*$ could be a constant function or a linear function. 
Clearly, the difficulty in learning $f^*$ via training neural networks should crucially depend on the properties of $f^*$ itself. 
We observe that the training convergence rate might be determined by how $f^*$ can be decomposed into the eigenspaces of the integral operator defined in \eqref{eq: int op}. 
This observation is also validated by a couple of existing empirical observations: (1) The spectrum of the MNIST data \cite{lecun1998gradient} concentrates on the first a few eigenspaces; and (2) the training is slowed down if labels are partially corrupted \cite{zhang2016understanding,arora2019fine}. 
Compared with \cite{arora2019fine}, we use spectral projection concentration to show how the random eigenvalues and the random projections in \cite[Eq.(8) in Theorem 4.1]{arora2019fine} are controlled by $f^*$ and $\rho$.

We first present a sufficient condition for the convergence of $\norm{\bm{y} -\hat{\bm{y}}(t)}$. 
\begin{theorem}[Sufficiency]
\label{thm: convergence rate under sufficient condition}
Let $0<\eta<1$. 
Suppose there exist $c_0\in (0,1)$ and $c_1>0$ such that 
\begin{align}
\label{eq: sufficiency}
\norm{\frac{1}{\sqrt{n}}\pth{\bm{I} - \eta\bm{K}}^{t} \bm{y} }\le (1-\eta c_0)^t   + c_1, ~~~ \forall ~ t.   
\end{align}
For any $\delta \in (0, \frac{1}{4})$ and given $T>0$, if 
\begin{align}
\label{eq: lower bound on overp}
m\ge ~ \frac{32}{c_1^2} \pth{\pth{\frac{1}{c_0} + 2\eta T c_1}^4 +4\log\frac{4n}{\delta}\pth{\frac{1}{c_0} + 2\eta T c_1}^2 },
\end{align}
then with probability at least $1-\delta$, the following holds for all $t \le T$:  
\begin{align}
\label{eq: real convergence}
\norm{\frac{1}{\sqrt{n}}\pth{\bm{y} -\hat{\bm{y}}(t)}} \le (1-\eta c_0)^t  + 2c_1. 
\end{align}
\end{theorem}
Theorem \ref{thm: convergence rate under sufficient condition} is proved in Appendix \ref{app: thm: convergence rate under sufficient condition}. %
 Theorem \ref{thm: convergence rate under sufficient condition} says that if $\norm{\frac{1}{\sqrt{n}}\pth{\bm{I} - \eta\bm{K}}^{t} \bm{y} }$ converges to 
$c_1$ exponentially fast, then $\norm{\frac{1}{\sqrt{n}}\pth{\bm{y} -\hat{\bm{y}}(t)}}$ converges to $2c_1$ with the same convergence rate guarantee provided that the neural network is sufficiently parametrized. 
Recall that $y_i\in [-1,1]$ for each $i\in [n]$. Roughly speaking, in our setup, $y_i = \Theta(1)$ and $\norm{\bm{y}} = \sqrt{\sum_{i=1}^n y_i^2} = \Theta(\sqrt{n})$. Thus we have the $\frac{1}{\sqrt{n}}$ scaling in \eqref{eq: sufficiency} and \eqref{eq: lower bound on overp} for normalization purpose. 

Similar results were shown in \cite{du2018gradient,arora2019fine} with $\eta= \frac{\lambda_{\min}(\bm{K})}{n}$, $c_0= n\lambda_{\min}(\bm{K})$ and $c_1=0$. But the obtained convergence rate $\log \frac{1}{1-\lambda_{\min}^2(\bm{K})} \to 0$ as $n\diverge$. 
In contrast, as can be seen later (in Corollary \ref{thm: main theorem zero approximation error}), if $f^*$ lies in the span of a small number of eigenspaces of the integral operator in  \eqref{eq: int op}, then we can choose $\eta=\Theta(1)$, choose $c_0$ to be a value that is determined by the target function $f^*$ and the distribution $\rho$ only, and choose $c_1=\Theta(\frac{1}{\sqrt{n}})$. Thus, the resulting convergence rate $\log \frac{1}{1-\eta c_0}$ does not approach 0 as $n\diverge$.  The additive term $c_1 =\Theta(1/\sqrt{n})$ arises from the fact that only finitely many data tuples are available.  
Both the proof of Theorem \ref{thm: convergence rate under sufficient condition} and the proofs in \cite{du2018gradient,arora2019fine,allen2018learning} are based on the observation that when the network is sufficiently over-parameterized, the sign changes (activation pattern changes) of the hidden neurons are sparse. 
Different from \cite{du2018gradient,arora2019fine}, our proof does not use $\lambda_{\min}(\bm{K})$
; see Appendix \ref{app: thm: convergence rate under sufficient condition} for details.

%
%

It remains to show, with high probability, \eqref{eq: sufficiency}  in Theorem \ref{thm: convergence rate under sufficient condition} holds with properly chosen $c_0$ and $c_1$.   
By the spectral theorem \cite[Theorem 4, Chapter X.3]{dunford1963linear} and \cite{rosasco2010learning}, 
$L_{\calK}$ has a spectrum with {\em distinct} eigenvalues $\mu_1 > \mu_2 > \cdots$ \footnote{ 
The sequence of distinct eigenvalues can possibly be of finite length.
In addition, the sequences of $\mu_i$'s and $\lambda_i$'s (in Theorem \ref{thm: vacuous rate}) are different, the latter of which consists of repetitions. 
}
such that 
\[
L_{\calK} = \sum_{i\ge1} \mu_i P_{\mu_i}, ~~~ \text{with} ~ ~ P_{\mu_i} := \frac{1}{2\pi {\bf \im}} \int_{\Gamma_{\mu_i}} (\gamma \calI - L_{\calK})^{-1} \diff \gamma, 
\]
where $P_{\mu_i}: L^2(\calS^{d-1},\rho) \to L^2(\calS^{d-1},\rho)$ is the {\em orthogonal projection operator} onto the eigenspace associated with eigenvalue $\mu_i$;  
here (1) $\im$ is the imaginary unit, 
and (2) the integral can be taken over any closed simple rectifiable curve (with positive direction) $\Gamma_{\mu_i}$ containing $\mu_i$ only and no other distinct eigenvalue. 
In other words, $P_{\mu_i}f$ is the function obtained by projecting function $f$ onto the eigenspaces of the integral operator $L_{\calK}$ associated with $\mu_i$. 
%
 
Given an $\ell \in \naturals$, let $m_{\ell}$ be the sum of the multiplicities of the first $\ell$ nonzero top eigenvalues of $L_{\calK}$. That is, $m_1$ is the multiplicity of $\mu_1$ and $(m_2-m_1)$ is the multiplicity of $\mu_2$.  
By definition, 
\[
\lambda_{m_{\ell}} = \mu_{\ell}\not= \mu_{\ell+1} = \lambda_{m_{\ell}+1}, ~~ \forall \,  \ell. 
\]
%
\begin{theorem}
\label{thm: main theorem}
For any $\ell \ge 1$ such that $\mu_i>0, ~ \text{for } i=1, \cdots, \ell$, let 
\[
\epsilon(f^*, \ell) : = \sup_{x\in \calS^{d-1}}\abth{f^* (x)- (\sum_{1\le i\le \ell}P_{\mu_i} f^*)(x)} 
\]
be the approximation error of the span of the eigenspaces associated with the first $\ell$ {\em distinct} eigenvalues. 
Then given $\delta \in (0, \frac{1}{4})$ and $T>0$, if 
$n>\frac{256\log \frac{2}{\delta}}{(\lambda_{m_{\ell}} -  \lambda_{m_{\ell}+1})^2}$ 
and 
\[
m\ge ~ \frac{32}{c_1^2} \pth{\pth{\frac{1}{c_0} + 2\eta T c_1}^4 +4\log\frac{4n}{\delta}\pth{\frac{1}{c_0} + 2\eta T c_1}^2 },
\]
then with probability at least $1-3\delta$, the following holds for all $t \le T$:  
\begin{align*}
\norm{\frac{1}{\sqrt{n}}\pth{\bm{y} -\hat{\bm{y}}(t)}} \le \pth{1- \frac{3}{4}\eta\lambda_{m_\ell}}^t  + 
\frac{16\sqrt{2}  \sqrt{\log \frac{2}{\delta} }}{(\lambda_{m_\ell} -  \lambda_{{m_\ell}+1}) \sqrt{n}}+ 2\sqrt{2}\epsilon(\ell, f^*). 
\end{align*}
\end{theorem}
Since $\lambda_{m_\ell}$ is determined by $f^*$ and $\rho$ only, with $\eta=1$, 
the convergence rate $\log \frac{1}{1-\frac{3}{4}\lambda_{m_\ell}}$ is constant w.\,r.\,t.\,$n$. 

\begin{remark}[Early stopping]
\label{rmk: early stopping}
In Theorems \ref{thm: convergence rate under sufficient condition} and \ref{thm: main theorem}, the derived lower bounds of $m$ grow in $T$. To control $m$, we need to terminate the GD training at some ``reasonable'' $T$. 
Fortunately, $T$ is typically small. To see this, 
note that $\eta$, $c_0$, and $c_1$ are independent of $t$. By \eqref{eq: sufficiency} and \eqref{eq: real convergence} we know $\norm{\frac{1}{\sqrt{n}}\pth{\bm{y} -\hat{\bm{y}}(t)}}$ decreases to $\Theta(c_1)$ in $(\log \frac{1}{c_1}/\log \frac{1}{1-\eta c_0})$ iterations provided that 
$(\log \frac{1}{c_1}/\log \frac{1}{1-\eta c_0}) \le T$.  
Thus, to guarantee $\norm{\frac{1}{\sqrt{n}}\pth{\bm{y} -\hat{\bm{y}}(t)}} = O(c_1)$,  it is enough to terminate GD at iteration $T = \Theta(\log \frac{1}{c_1}/\log \frac{1}{1-\eta c_0})$.  
Similar to us, early stopping is adopted in \cite{allen2018learning,li2019gradient}, and is commonly adopted in practice. 
\end{remark}

\begin{corollary}[zero--approximation error]
\label{thm: main theorem zero approximation error}
Suppose there exists $\ell$ such that $\mu_i>0, ~ \text{for } i=1, \cdots, \ell$, and 
$\epsilon(f^*, \ell)= 0.$
Then let 
$\eta=1$ and $T= \frac{\log n}{-\log (1-\frac{3}{4}\lambda_{m_{\ell}})}$. 
For a given $\delta \in (0, \frac{1}{4})$, if 
$n> \frac{256\log \frac{2}{\delta}}{(\lambda_{m_{\ell}} -  \lambda_{m_{\ell}+1})^2}$ 
and 
$m\gtrsim \pth{n\log n} \pth{\frac{1}{\lambda_{m_{\ell}}^4} + \frac{\log^4 n \log^2\frac{1}{\delta}}{(\lambda_{m_{\ell}} -  \lambda_{m_{\ell}+1})^2 n^2 \lambda_{m_{\ell}}^4}}$, 
then with probability at least $1-3\delta$,  the following holds for all $t \le T$:
\begin{align*}
\norm{\frac{1}{\sqrt{n}}\pth{\bm{y} -\hat{\bm{y}}(t)}} \le (1-\frac{3\lambda_{m_{\ell}}}{4})^t  + \frac{16\sqrt{2\log 2/\delta}}{\sqrt{n}\pth{\lambda_{m_{\ell}} -  \lambda_{m_{\ell}+1}}}.
\end{align*}
\end{corollary}
Corollary \ref{thm: main theorem zero approximation error} says that for fixed $f^*$ and fixed distribution $\rho$, nearly-linear network over-parameterization $m=\Theta(n\log n)$ is enough for GD method to converge exponentially fast as long as $\frac{1}{\delta} = O(\poly(n))$. 
Corollary \ref{thm: main theorem zero approximation error} follow immediately from Theorem \ref{thm: main theorem} by specifying the relevant parameters such as $\eta$ and $T$. 
To the best of our knowledge, this is the first result showing sufficiency of nearly-linear network over-parameterization. 
%
Note that $(\lambda_{m_\ell} -  \lambda_{{m_\ell}+1})>0$ is the eigengap between the $\ell$--th and $(\ell+1)$--th largest distinct eigenvalues of the integral operator, and is irrelevant to $n$. 
 Thus, for fixed $f^*$ and $\rho$,  $c_1 = \Theta \pth{\sqrt{\log \frac{1}{\delta}/n }}$.

\section{Application to Uniform Distribution and Polynomials}
We illustrate our general results by applying them to the setting where the target functions are polynomials and the feature vectors are uniformly distributed on the sphere $\calS^{d-1}$. 

Up to now, we implicitly incorporate the bias $b_j$ in $w_j$ by augmenting the original $w_j$; correspondingly, the data feature vector is also augmented. In this section, as we are dealing with distribution on the original feature vector, we explicitly separate out the bias from $w_j$. In particular, let $b_j^0 \sim \calN(0, 1)$. 
For ease of exposition, with a little abuse of notation, we use $d$ to denote the dimension of the $w_j$ and $x$ before the above mentioned augmentation. 
With bias, \eqref{eq: two layer model} can be rewritten as 
$f_{\bm{W}, \bm{b}} (x) = \frac{1}{\sqrt{m}}\sum_{j=1}^m a_j\qth{\iprod{x}{w_j} + b_j}_+, $
where $\bm{b} = \pth{b_1, \cdots, b_m}$ are the bias of the hidden neurons, and the kernel function in \eqref{eq: initial kernel random} becomes  
\begin{align}
\label{eq: kernel with bias}
\calK(x,s)  =  \frac{\iprod{x}{s} +1}{2\pi} \pth{\pi -\arccos \pth{\frac{1}{2} \pth{\iprod{x}{s} +1}}} ~~~ \forall ~ x, s \in \calS^{d-1}. 
\end{align}

From Theorem \ref{thm: main theorem} we know the convergence rate is determined by the eigendecomposition of the target function $f^*$ w.\,r.\,t.\,the eigenspaces of $L_{\calK}$. 
When $\rho$ is the uniform distribution on $\calS^{d-1}$,  the eigenspaces of $L_{\calK}$ are the spaces of homogeneous harmonic polynomials, denoted by $\calH^{\ell}$ for $\ell \ge 0$. Specifically, $L_{\calK} =    \sum_{\ell\ge 0} \beta_{\ell} P_{\ell},$
where $P_{\ell}$ (for $\ell \ge 0$) is the orthogonal projector onto $\calH^{\ell}$ and $\beta_{\ell} =  \frac{\alpha_\ell \frac{d-2}{2}}{\ell + \frac{d-2}{2}}>0$ is the associated eigenvalue -- $\alpha_{\ell}$ is the coefficient of $\calK(x,s)$ in the expansion into Gegenbauer polynomials. 
Note that $\calH^{\ell}$ and $\calH^{\ell^{\prime}}$ are orthogonal when $\ell \not=\ell^{\prime}$. See appendix \ref{app: Harmonic analysis: gegenbauer} for relevant backgrounds on harmonic analysis on spheres.

Explicit expression of  eigenvalues $\beta_{\ell}>0$ is available;   
see Fig.\,\ref{fig:monotone} for an illustration of $\beta_\ell$. In fact, there is a line of work on efficient computation of the coefficients of Gegenbauer polynomials expansion \cite{cantero2012rapid}. 

If the target function $f^*$ is a standard polynomial of degree $\ell^*$, by \cite[Theorem 7.4]{yiwang2014}, we know $f^*$ can be perfectly projected onto the direct sum of the spaces of homogeneous harmonic polynomials up to degree $\ell^*$. The following corollary follows immediately from Corollary \ref{thm: main theorem zero approximation error}.  
\begin{corollary}
\label{cor: uniform}
Suppose $f^*$ is a degree $\ell^*$ polynomial, and the feature vector $x_i$'s are $\iid$ generated from the uniform distribution over $\calS^{d-1}$.  
%
Let 
$\eta=1$, and $T= \Theta(\log n)$. 
For a given $\delta \in (0, \frac{1}{4})$, if 
$n = \Theta\pth{\log \frac{1}{\delta}}$  
and $m = \Theta(n\log n \log^2 \frac{1}{\delta})$,
then with probability at least $1-\delta$,  the following holds for all $t \le T$:
\begin{align*}
\norm{\frac{1}{\sqrt{n}}\pth{\bm{y} -\hat{\bm{y}}(t)}} \le \pth{1-\frac{3c_0}{4}}^t  + \Theta(\sqrt{\frac{\log 1/\delta}{n}}),  ~~~~ \text{where }c_0 =\min \sth{\beta_{\ell^*}, \beta_{\ell^*+1}}. 
\end{align*}
\end{corollary}
For ease of exposition, in the above corollary, $\Theta(\cdot)$ hides dependence on quantities such as eigengaps -- as they do not depend on $n$, $m$, and $\delta$. 
Corollary \ref{cor: uniform} and $\beta_\ell$ in Fig.\,\ref{fig:monotone} together suggest that the convergence rate decays with both the dimension $d$ and the polynomial degree $\ell$. 
This is validated in Fig.\,\ref{fig:monotone}. 
It might be unfair to compare the absolute values of training errors since $f^*$ are different. Nevertheless, the convergence rates can be read from slope in logarithmic scale. We see that the convergence slows down as $d$ increases, and learning a quadratic function is slower than learning a linear function.
\begin{figure}
	\centering
	\begin{subfigure}[t]{0.48\textwidth}
		\centering
		\includegraphics[width=\textwidth]{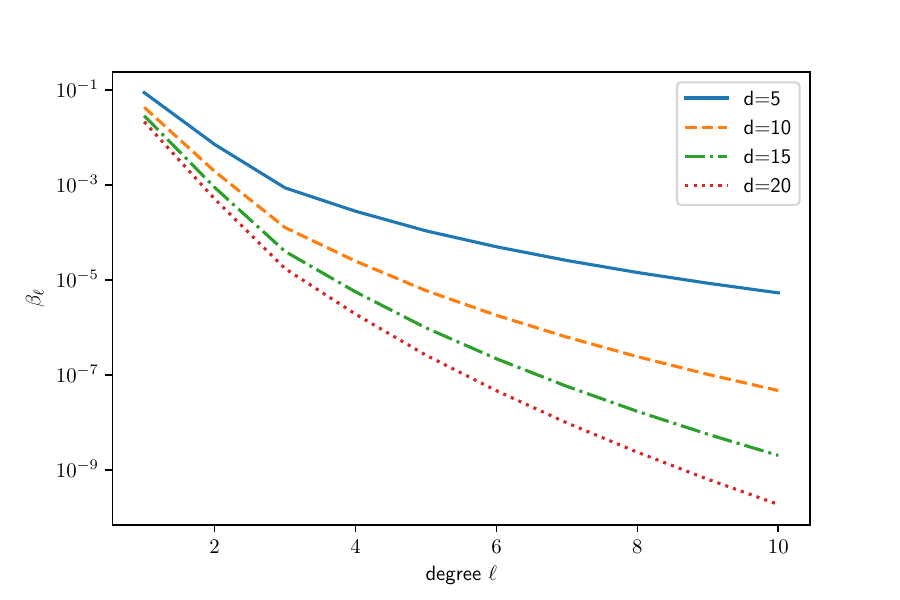}
		\caption{Plot of $\beta_{\ell}$ with $\ell$ under different $d$. Here, the $\beta_{\ell}$ is monotonically decreasing in $\ell$. 
		}
		\label{fig:monotone}
	\end{subfigure}
	\hfill 
	\begin{subfigure}[t]{0.48\textwidth}
		\centering
		\includegraphics[width=\textwidth]{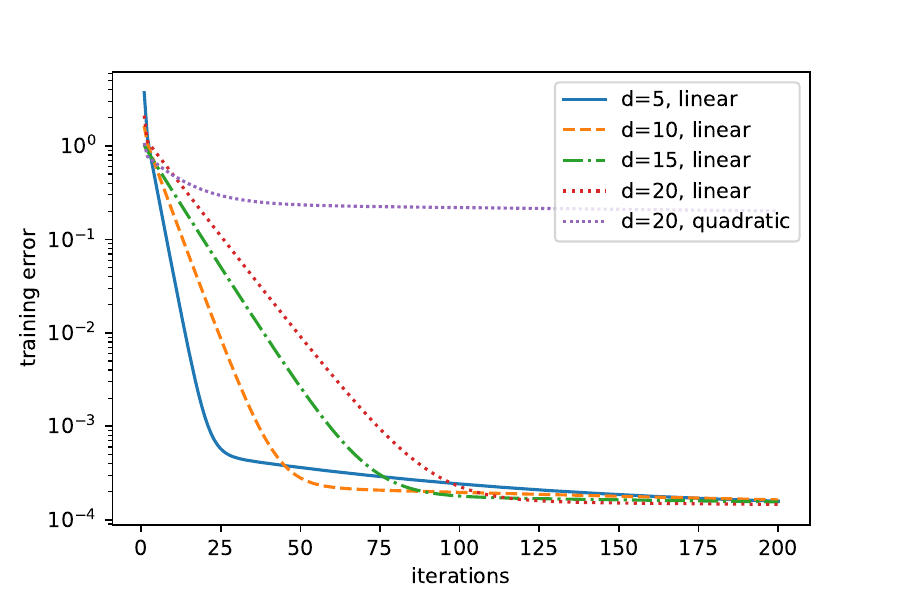}
		\caption{Training with $f^*$ being randomly generated linear or quadratic functions with $n=1000$, $m=2000$.\label{converge}} 
	\end{subfigure}
	\caption{Application to uniform distribution and polynomials.}
\end{figure}

Next we present the explicit expression of $\beta_{\ell}$. 
For ease of exposition, let $h(u) : = \calK(x,s)$ where $u=\iprod{x}{s}$. 
By \cite[Eq.\,(2.1) and Theorem 2]{cantero2012rapid}, we know 
\begin{align}
\label{eq: eigen compute 1}
\beta_{\ell} = \frac{d-2}{2} \sum_{k=0}^{\infty} \frac{h_{\ell+2k}}{2^{\ell+2k}k!\pth{\frac{d-2}{2}}_{\ell+k+1}},
\end{align}
where $h_{\ell} :=h^{(\ell)}(0)$ is the $\ell$--th order derivative of $h$ at zero, and the {\em Pochhammer symbol} $(a)_k$ is defined recursively as $(a)_0=1$, $(a)_k = (a+k-1)(a)_{k-1}$ for $k\in \naturals$. 
By a simple induction, it can be shown that $h_{0} = h^{(0)}(0) =1/3$, and for $k\ge 1$, 
\begin{align}
\label{eq: derivatives}
h_k = \frac{1}{2} \indc{k=1}  - \frac{1}{\pi 2^k} \pth{k \pth{\arccos 0.5}^{(k-1)} + 0.5 \pth{\arccos 0.5}^{(k)}}, 
\end{align}
where the computation of the higher-order derivative of $\arccos$ is standard.  It follows from \eqref{eq: eigen compute 1} and \eqref{eq: derivatives} that $\beta_{\ell} >0$, and 
$\beta_{2\ell} > \beta_{2(\ell +1)}$ and $\beta_{2\ell +1} > \beta_{2\ell +3}$ for all $\ell \ge 0$. However,  an analytic order among $\beta_\ell$ is unclear, and we would like to explore this in the future. 

\section*{Acknowledgement}
We would like to thank Yang Yuan (Tsinghua IIIS) for his insightful initial discussions,  and Rong Ge (Duke) for suggesting the network initialization rule.

\bibliographystyle{alpha}
\bibliography{2NNBib}

\newpage 

\begin{center}
\bf \Large Appendices
\end{center}

\appendix 
%
%

\section{Existing Generalization Bound}
\label{app: generalization bound}
Though a generalization bound is given in \cite[Theorem 5.1 and Corollary 5.2]{arora2019fine}, it is unclear how their bound scales in $n$. In particular, the dominating term of their bound is $\sqrt{\frac{2\bm{y}^{\top}\pth{n\bm{H}}^{-1} \bm{y}}{n}}$. Here, the matrix $\bm{H}$ is defined w.r.t. the training dataset $\sth{(x_i, y_i): 1\le i \le n}$ and the $\ell_2$ norm of the response vector $\bm{y}$ grows with $n$. As a result of this, the scaling of the magnitude of $\bm{y}^{\top}\pth{n\bm{H}}^{-1} \bm{y}$ in $n$ is unclear. Recall that $y_i = \Theta (1)$ for $i=1, \cdots, n$; thus, $\norm{\bm{y}} = \Theta(\sqrt{n})$. 
If we do not care about the structure of the target function $f^*$ and 
allow $\frac{\bm{y}}{\sqrt{n}}$ to be the eigenvector associated with the least eigenvalue of $\bm{H}$, then $\sqrt{\frac{2\bm{y}^{\top}\pth{n\bm{H}}^{-1} \bm{y}}{n}}$ might not decrease to zero as $n\diverge$. 
This is because %
\[
\sqrt{\frac{2\bm{y}^{\top}\pth{n\bm{H}}^{-1} \bm{y}}{n}}  = \sqrt{2\pth{\frac{\bm{y}}{\sqrt{n}}}^{\top}\pth{n\bm{H}}^{-1} \pth{\frac{\bm{y}}{\sqrt{n}}}} = \Theta\pth{ 
\pth{\lambda_{\min} \pth{n\bm{H}}}^{-\frac{1}{2}}}.
\]
 As illustrated by Fig.\,\ref{fig:eigen-min-inverse}, even when $\rho$ is the uniform distribution,  $\pth{\lambda_{\min} \pth{n\bm{H}}}^{-\frac{1}{2}}$ does not approach zero as $n$ increases. In general, without specifying the structure of the target function $f^*$, in the presence of the randomness of data generation and the network initialization, it is unclear which eigenvalues of $\bm{H}$ determines the generalization capability of the learned neural network.  

\begin{figure}[hb]
  \centering
    \includegraphics[width=0.6\textwidth]{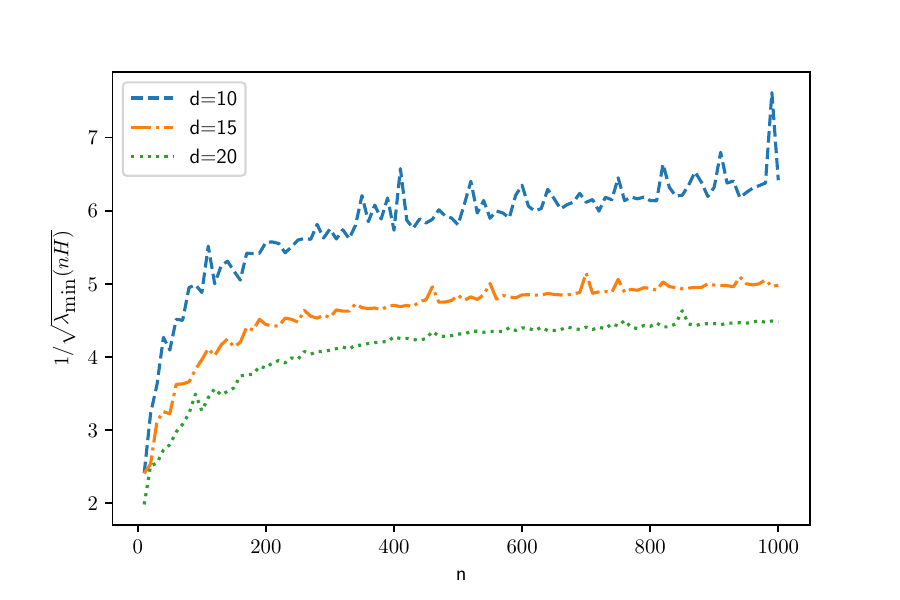}
    
      \caption{
               Plot of  $(\lambda_{\min}(n\bm{H}))^{-\frac{1}{2}}$ under different sample sizes. 
         Here the feature vectors are generated from the uniform distribution on the unit sphere. 
      }
      \label{fig:eigen-min-inverse}
\end{figure}

\section{Proof of Theorem  \ref{thm: matrix representation}}
\label{app: Proof of matrix representation}

We use the following proposition in proving Theorem  \ref{thm: matrix representation}. 

\begin{proposition}
\label{prop: single neuron bound}
It is true that for any 
$j\in [m]$, 
$i\in [n]$, and 
$t\ge 0$,  
\begin{align}
\nonumber
&\frac{\eta a_j}{n\sqrt{m}} \sum_{i^{\prime}=1}^n (y_{i^{\prime}}-\hat y_{i^{\prime}}(t)) \Iprod{x_i}{x_{i^{\prime}}}\indc{\iprod{ w_j^{t}}{x_{i^{\prime}}}  >0} \indc{\iprod{ w_j^{t}}{x_i}  >0}  ~~~ \label{eq: neuron lb}\\
& \qquad \qquad \qquad  \qquad \le \qth{\Iprod{w_{j}^{t+1}}{x_i}  } _+- \qth{\Iprod{w_{j}^{t}}{x_i} }_+\\
& \qquad \qquad \qquad \qquad  \le \frac{\eta a_j}{n\sqrt{m}} \sum_{i^{\prime}=1}^n (y_{i^{\prime}}-\hat y_{i^{\prime}}(t))\Iprod{x_i}{x_{i^{\prime}}}\indc{\iprod{ w_j^{t}}{x_{i^{\prime}}}  >0} \indc{\iprod{ w_j^{t+1}}{x_i} >0}. ~ \label{eq: neuron ub}
\end{align}
\end{proposition}
\begin{proof}
From \eqref{eq: weight update new}, 
we have 
\begin{align}
\label{eq: lm: single neuron}
\Iprod{w_j^{t+1}}{x_i}-\Iprod{w_j^t}{x_i} =\frac{\eta a_j}{n\sqrt{m}}\sum_{i^{\prime}=1}^n (y_{i^{\prime}}-\hat y_{i^{\prime}}(t)) \Iprod{x_i}{x_{i^{\prime}}}\indc{\Iprod{w_j^t}{x_{i^{\prime}}}>0}.
\end{align}
Then the conclusion follows from the fact that 
\begin{align}
\label{eq: 1-lip continuous}
\indc{a>0}(b-a)\le ~ \qth{b}_+-\qth{a}_+ ~ \le \indc{b>0}(b-a), ~~ \forall ~ a, b. 
\end{align}

\end{proof}

\begin{remark}
The inequality in \eqref{eq: 1-lip continuous} can be extended to 
a general family of activation function $\sigma$ if 
\[
\sigma'(a)(b-a)\le \sigma(b)-\sigma(a)\le \sigma'(b)(b-a),\quad \forall~a,b,
\]
where $\sigma^{\prime}(\cdot)$ is the derivative of $\sigma$. For ReLu activation, the right derivative is used. 
\end{remark}

\begin{proof}[\bf Proof of Theorem  \ref{thm: matrix representation}]
Recall from \eqref{eq: prediction at time t} that for $t\ge 0$, 
\begin{align*}
\hat{y}_i(t+1) &= \frac{1}{\sqrt{m}} \sum_{j=1}^m a_j\qth{\iprod{w_{j}^{t+1}}{x}}_+
 = \frac{1}{\sqrt{m}} \sum_{j\in \calA}  \qth{\iprod{w_{j}^{t+1}}{x}}_+ - \frac{1}{\sqrt{m}} \sum_{j\in \calB}\qth{\iprod{w_{j}^{t+1}}{x}}_+. 
\end{align*}
Thus,
\begin{align*}
\hat{y}_i(t+1) - \hat{y}_i(t)  & = \frac{1}{\sqrt{m}} \sum_{j\in \calA}  \pth{\qth{\iprod{w_{j}^{t+1}}{x}}_+ - \qth{\iprod{w_{j}^{t}}{x}}_+}  - \frac{1}{\sqrt{m}} \sum_{j\in \calB}\pth{\qth{\iprod{w_{j}^{t+1}}{x}}_+ - \qth{\iprod{w_{j}^{t}}{x}}_+}\\
& \overset{(a)}{\le} \frac{\eta }{nm} \sum_{j\in \calA} \sum_{i^{\prime}=1}^n (y_{i^{\prime}}-\hat y_{i^{\prime}}(t))\Iprod{x_i}{x_{i^{\prime}}}\indc{\iprod{ w_j^{t}}{x_{i^{\prime}}}  >0} \indc{\iprod{ w_j^{t+1}}{x_i} >0}\\
& \quad + \frac{\eta }{nm} \sum_{j\in \calB}\sum_{i^{\prime}=1}^n (y_{i^{\prime}}-\hat y_{i^{\prime}}(t))\Iprod{x_i}{x_{i^{\prime}}}\indc{\iprod{ w_j^{t}}{x_{i^{\prime}}}  >0} \indc{\iprod{ w_j^{t}}{x_i}  >0}\\
& =  \eta  \sum_{i^{\prime}=1}^n  \pth{\tilde{\bm{H}}^+_{i i^{\prime}}(t+1) +  \bm{H}^-_{i i^{\prime}}(t+1)} (y_{i^{\prime}}-\hat y_{i^{\prime}}(t)),
\end{align*}
where inequality (a) follows from Proposition \ref{prop: single neuron bound}. 
Thus, 
\begin{align*}
y_i - \hat{y}_i(t+1) \ge y_i - \hat{y}_i(t) - \eta \sum_{i^{\prime}=1}^n  \pth{\tilde{\bm{H}}^+_{i i^{\prime}}(t+1) +  \bm{H}^-_{i i^{\prime}}(t+1) } (y_{i^{\prime}}-\hat y_{i^{\prime}}(t)),
\end{align*}
whose matrix form is 
$\pth{\bm{y} - \bm{\hat{y}}(t+1)} \ge 
\pth{\bm{I}- \eta\pth{\tilde{\bm{H}}^+(t+1) + \bm{H}^-(t+1)}} \pth{\bm{y} - \bm{\hat{y}}(t)}, 
$
 proving the lower bound in Theorem  \ref{thm: matrix representation}.  
The upper bound in Theorem  \ref{thm: matrix representation} can be obtained analogously. 

\end{proof}

\section{Proof of Theorem \ref{thm: vacuous rate}}
\label{app: thm: vacuous rate}
Let $\lambda_1\ge \lambda_2\ge \dots$ be the spectrum of $L_\calK$ defined in \eqref{eq: int op}, whose existence is given by the spectral theorem \cite[Theorem 4, Chapter X.3]{dunford1963linear}. 
Recall that 
\[
\bm{H}_{ii^{\prime}} = \frac{1}{nm}\iprod{x_i}{x_{i^{\prime}}}\sum_{j=1}^m \indc{\iprod{w_j^0}{x_{i^{\prime}}}>0} \indc{\iprod{w_j^0}{x_{i}}>0}
\]
is a random $n\times n$ matrix, where the randomness comes from (1) the data randomness $\sth{(x_i, y_i): 1\le i\le n}$ and (2) the network initialization randomness. Thus, $\tilde{\lambda}_i$ for $1\le i \le n$ are random. Notably, $\bm{K}=\expect{\bm{H}}$ is still random as the data randomness remains. Denote 
the spectrum of $\bm{K}$ as $\hat{\lambda}_1 \ge \cdots \ge \hat{\lambda}_n$. By \cite[Proposition 10]{rosasco2010learning}, with probability at least $1-\frac{\delta}{2}$ over data generation, 
\begin{equation}
\label{eq:eigen-concentration-n}
\sup_i |\lambda_i - \hat{\lambda}_i| \le \sqrt{\frac{8 \log (4/\delta)}{n}}. 
\end{equation}
For a given dataset $x_1,\dots,x_n\in\calS^{d-1}$, by Hoeffding's inequality and the union bound, with probability at least $1-\frac{\delta}{2}$ over network initialization, 
\begin{equation}
	\label{eq:Hoeffding-H}
\Fnorm{\bm{H} - \bm{K}} = \Fnorm{\bm{H} - \Expect \bm{H}}  \le \sqrt{\frac{\log (4n^2/\delta)}{m}}. 	
\end{equation}
Then, it follows from Weyl's inequality that 
\begin{equation}
\label{eq:eigen-concentration-m}
\sup_i|\tilde\lambda_i - \hat{\lambda}_i|
\le  \sqrt{\frac{\log (4n^2/\delta)}{m}}.
\end{equation}
We conclude \eqref{eq:eigen-concentration} by combining \eqref{eq:eigen-concentration-n} and \eqref{eq:eigen-concentration-m}.
Letting $\delta=\frac{1}{n}$, we have, with probability $1-\frac{1}{n}$, 
\[
0\le \lambda_{\min}(\bm{H}) \le \lambda_n+\sqrt{\frac{\log (4n^3)}{m}}+\sqrt{\frac{8 \log (4 n)}{n}},
\]
where the right-hand side vanishes with $n$.
Thus, let $n\diverge$, we have $1-\frac{1}{n} \to 1$, and  
\begin{align*}
\lim_{n\diverge}\lambda_n+ \lim_{n\diverge} \sqrt{\frac{\log (4n^3)}{m}}+\lim_{n\diverge}\sqrt{\frac{8 \log (4 n)}{n}}
= 0,
\end{align*}
proving the theorem. 

\section{Proof of Theorem  \ref{thm: convergence rate under sufficient condition}}
\label{app: thm: convergence rate under sufficient condition} 
We prove Theorem  \ref{thm: convergence rate under sufficient condition} via two steps: 
(1) We first 
bound the perturbation terms. 
(2) Then, we prove Theorem  \ref{thm: convergence rate under sufficient condition} via an induction argument. 

\subsection{Bounding the perturbation}
For ease of exposition, let 
\begin{align}
\bm{H}(t) &: = \pth{\bm{H}^+(t) + \bm{H}^-(t)}, \label{eq: symmetric matrix} \\
\bm{M}(t) &: = 
\pth{\tilde{\bm{H}}^-(t)- \bm{H}^-(t)}, \label{eq: upper bound matrix}\\
\bm{L}(t) &: = \pth{\tilde{\bm{H}}^+(t)- \bm{H}^+(t)}. \label{eq: lower bound matrix}
\end{align}

\begin{lemma}
\label{lm: Bounding the perturbation 1}
Choose $0< \eta < 1$. Then for any $t\ge 0$, it holds that 
\begin{align*}
\norm{\bm{y} -\hat{\bm{y}}(t+1)} &\le \norm{\pth{\bm{I} -\eta \bm{K}}^{t+1}\bm{y}} 
 + \eta\sum_{r=2}^{t+2}  \norm{\pth{\bm{K} - \bm{H}(r-1)}} \norm{\pth{\bm{I} -\eta \bm{K}}^{r-2} \bm{y}}\\
& \quad + \eta \sum_{r=2}^{t+2} \pth{\norm{ \bm{M}(r-1)} +  \norm{ \bm{L}(r-1)}} \norm{\pth{\bm{y} -\hat{\bm{y}}(r-2)}}.  
\end{align*}
\end{lemma}
\begin{proof}
Let $\bm{\epsilon}(t+1) : = \pth{\bm{y} -\hat{\bm{y}}(t+1)} - \pth{\bm{I} - \eta \bm{H}(t+1)}\pth{\bm{y} -\hat{\bm{y}}(t)}$, for $t\ge 0$, i.e., 
\begin{align}
\label{eq: rewrite error}
\bm{y} -\hat{\bm{y}}(t+1)  = \pth{\bm{I} - \eta \bm{H}(t+1)}\pth{\bm{y} -\hat{\bm{y}}(t)} + \bm{\epsilon}(t+1),  ~~~ \forall \, t\ge 0. 
\end{align}
It follows from Theorem \ref{thm: matrix representation} that 
\begin{align}
\label{eq: error bound}
\norm{\bm{\epsilon}(t+1)}\le \eta(\norm{\bm{M}(t+1)}+\norm{\bm{L}(t+1)})\norm{\bm{y} -\hat{\bm{y}}(t)}.
\end{align}
Expanding \eqref{eq: rewrite error} over $t$, we have 
\begin{align}
\label{eq: original}
\bm{y} -\hat{\bm{y}}(t+1) &= \qth{\prod_{r=1}^{t+1} \pth{\bm{I} - \eta \bm{H}(r)}}\pth{\bm{y} -\hat{\bm{y}}(0)} +  \sum_{r=2}^{t+2} \qth{\prod_{k=r}^{t+1} \pth{\bm{I} - \eta \bm{H}(k)}} \bm{\epsilon}(r-1),
\end{align}
where $\prod_{r=k}^{t+1}\pth{\bm{I} -\eta\bm{H}(r)} : = \pth{\bm{I} -\eta\bm{H}(t+1)} \times \cdots \times \pth{\bm{I} -\eta\bm{H}(k)}$ for $k\le t+1$ is a  backward matrix product, and $\prod_{k=t+2}^{t+1}\pth{\bm{I} -\eta\bm{H}(k)} : = \bm{I}$. Recall that $\hat{\bm{y}}(0) = \bm{0}$. Eq.\,\eqref{eq: original} can be simplified as 
\begin{align*}
\bm{y} -\hat{\bm{y}}(t+1) &= \qth{\prod_{r=1}^{t+1} \pth{\bm{I} - \eta \bm{H}(r)}}\bm{y}  +  \sum_{r=2}^{t+2} \qth{\prod_{k=r}^{t+1} \pth{\bm{I} - \eta \bm{H}(k)}} \bm{\epsilon}(r-1). 
\end{align*}

In addition, it can be shown by a simple induction that 
\begin{align*}
\qth{\prod_{r=1}^{t+1}\pth{\bm{I} - \eta\bm{H}(r)}} = \pth{\bm{I} -\eta \bm{K}}^{t+1}  +  \eta\sum_{r=2}^{t+2} \qth{\prod_{k=r}^{t+1} \pth{\bm{I} - \eta \bm{H}(k)}} \pth{\bm{K} - \bm{H}(r-1)} \pth{\bm{I} -\eta \bm{K}}^{r-2}. 
\end{align*}
Thus, we have 
\begin{align*}
\bm{y} -\hat{\bm{y}}(t+1)& =  \pth{\bm{I} -\eta \bm{K}}^{t+1} \bm{y} 
 + \eta\sum_{r=2}^{t+2} \qth{\prod_{k=r}^{t+1} \pth{\bm{I} - \eta \bm{H}(k)}} \pth{\bm{K} - \bm{H}(r-1)} \pth{\bm{I} -\eta \bm{K}}^{r-2}\bm{y}   \\
& \quad + \sum_{r=2}^{t+2} \qth{\prod_{k=r}^{t+1} \pth{\bm{I} - \eta \bm{H}(k)}} \bm{\epsilon}(r-1). 
\end{align*}
Notably, $\norm{\bm{H}(k)}^2 \le \Fnorm{\bm{H}(k)}^2  \le 1$  for each $k\ge 1$. 
Choosing  $0< \eta <1$, we have  $\norm{\bm{I} - \eta  \bm{H}(k)} \le 1$ for $ k \ge 1$. 
With this fact and \eqref{eq: error bound}, we conclude Lemma \ref{lm: Bounding the perturbation 1}.

\end{proof}

\vskip \baselineskip 

For each $i\in [n]$ and $t\ge 0$, let 
\begin{align}
\label{eq: ever sign change}
\calF(x_i, t) : = \sth{j: \, \exists 0\le k \le t ~ s.\ t. \ \indc{\iprod{w_j^{k}}{x_i} >0} \not= \indc{\iprod{w_j^0}{x_i} >0}}. 
\end{align}
be the set of hidden neurons that have ever flipped their signs by iteration $t$.

\begin{lemma}
\label{lmm:bound-ML}
Choose $0< \eta <1$. The following holds for all $t\ge 0$: 
\[
\max\sth{\norm{\bm{M}(t)}+\norm{\bm{L}(t)},~ \norm{\bm{H}-\bm{H}(t)}}\le \sqrt{\frac{4}{m^2n}\sum_{i=1}^n  \abth{\calF(x_i, t)}^2}.
\]
\end{lemma}
%
%
%
\begin{proof}

We bound $\bm{M}(t)$ as 
\begin{align}
\label{eq: bound norm of M}
\nonumber
\norm{\bm{M}(t)}^2 &\le \Fnorm{\bm{M}(t)}^2 = \sum_{i=1}^n \sum_{i^{\prime}=1}^n \bm{M}_{ii^{\prime}}^2(t)\\
\nonumber
& \le \frac{1}{m^2n^2}  \sum_{i=1}^n \sum_{i^{\prime}=1}^n \pth{\iprod{x_i}{x_{i^{\prime}}}\sum_{j\in \calF(x_i, t) \cap \calA} \indc{\iprod{w_j^t}{x_{i^{\prime}}}+b_j^t} }^2\\
& \le \frac{1}{m^2n}\sum_{i=1}^n  \abth{\calF(x_i, t)}^2. 
\end{align}
Similarly, $\norm{\bm{L}(t)}^2 \le \frac{1}{m^2n}\sum_{i=1}^n  \abth{\calF(x_i, t)}^2$ and $\norm{\bm{H} - \bm{H}(t)}^2  \le \frac{4}{m^2 n} \sum_{i=1}^n \abth{\calF(x_i, t)}^2$. %
%
%
%
\end{proof}

\begin{lemma}
\label{lm: concentration of initialization}
Fix a dataset $\sth{(x_i, y_i): \, 1\le i\le n}$. 
For any $R>0$ and $\delta \in (0, \frac{1}{4})$, with probability at least $1-\delta$ over network initialization,  
\begin{align}
\label{eq: uniform concentration}
\norm{\bm{K} - \bm{H}} + \sqrt{\frac{4}{m^2n} \sum_{i=1}^n \pth{\sum_{j=1}^m \indc{\abth{\iprod{w_j^0}{x_i}} \le R}}^2} 
\le 
\frac{4R}{\sqrt{2\pi}} + 4\sqrt{ \frac{\log (4n/\delta) }{m}}.
\end{align}
\end{lemma}
\begin{proof}
Since  $w_j^0 \sim \calN(\bm{0}, \bm{I})$ and $x_i\in \calS^{d-1}$, it is true that  
$\iprod{w_j^0}{x_i}  \sim \calN(0,1)$. 
Thus,  $\expect{\indc{\abth{\iprod{w_j^0}{x_i} }  \le R}} = \prob{\abth{\iprod{w_j^0}{x_i}} \le R} 
< \frac{2R}{\sqrt{2\pi}}$ holds for any $R>0$. 
By Hoeffding's inequality and union bound, we have, 
with probability at least $1-\frac{\delta}{2}$, 
\begin{align}
\label{eq: concentration1}
\frac{1}{m^2n} \sum_{i=1}^n \pth{\sum_{j=1}^m  \indc{\abth{\iprod{w_j^0}{x_i} } \le \,R} }^2 
&\le \pth{\frac{2R}{\sqrt{2\pi}} + \sqrt{ \frac{\log (4n/\delta) }{m}}}^2. 
\end{align}
In addition, we have shown in \eqref{eq:Hoeffding-H} that with probability at least $1-\frac{\delta}{2}$, 
\begin{equation*}
	\norm{\bm{H} - \bm{K}}  \le \sqrt{\frac{\log (4n^2/\delta)}{m}}. 
\end{equation*}
From \eqref{eq: concentration1} and \eqref{eq:Hoeffding-H}, we conclude Lemma \ref{lm: concentration of initialization}. 
%
\end{proof}

\subsection{Finishing the proof of Theorem \ref{thm: convergence rate under sufficient condition}}
%

For any $\delta \in (0, \frac{1}{4})$, let $\calE$ be the event on which \eqref{eq: uniform concentration} holds for $R=\frac{1}{\sqrt{m}}\pth{\frac{1}{c_0} + 2\eta T c_1}$. By Lemma \ref{lm: concentration of initialization}, we know $\prob{\calE} \ge 1-\delta$. 


Conditioning on event $\calE$ occurs, we finish proving Theorem \ref{thm: convergence rate under sufficient condition} via induction. 
Since we assume $\calE$ has occurred, all the relevant quantities below are deterministic. 
The base case $t=0$ trivially holds. Suppose \eqref{eq: real convergence} is true up to $t\le T-1$, and it suffices to prove it for $t+1$.
By 
\eqref{eq: sufficiency},  
we have 
\begin{equation}
\label{eq:sum1}
\eta\sum_{r=0}^{t} \norm{\frac{1}{\sqrt{n}}\pth{\bm{I} -\eta \bm{K}}^{r}\bm{y} } \le \eta\sum_{r=0}^{t} \pth{\pth{1- \eta c_0}^r +c_1 } \le  \frac{1}{c_0} + \eta T c_1.
\end{equation}
By the induction hypothesis, we have
\begin{equation}
\label{eq:sum2}
\eta\sum_{r=0}^{t}   \norm{\frac{1}{\sqrt{n}}\pth{\bm{y} -\hat{\bm{y}}(r)}} \le 
\eta \sum_{r=0}^{t} (\pth{1- \eta c_0}^r +2c_1)\le \frac{1}{c_0} + 2\eta T c_1.
\end{equation}
Also, we have $\abth{\calF(x_i, r) } \le \abth{\calF(x_i, t+1)}$ for each $r \le t+1$ by monotonicity.
Then, applying the upper bounds in Lemma \ref{lmm:bound-ML} and \eqref{eq:sum1} -- \eqref{eq:sum2} into Lemma \ref{lm: Bounding the perturbation 1}, we obtain that,
\begin{align}
\norm{\frac{1}{\sqrt{n}}\pth{\bm{y} -\hat{\bm{y}}(t+1)}} 
& \le \pth{\pth{1- \eta c_0}^{t+1} +c_1 }\nonumber\\
& \quad + \pth{\norm{\bm{K} - \bm{H}} + \sqrt{\frac{4}{m^2 n} \sum_{i=1}^n \abth{\calF(x_i, t+1)}^2}} \pth{\frac{2}{c_0}+3\eta Tc_1}.\label{eq:induction-next} 
\end{align}
It remains to bound the cardinality of $\calF(x_i, t+1)$. Note that
\[
\calF(x_i, t+1)\subseteq \sth{j:\abth{\iprod{w_j^0}{x_i} } \le \, \max_{1\le k \le t+1} \norm{w_j^k -w_j^0}}.
\]
Since $\norm{x_i}=1$, it follows from \eqref{eq: weight update new} that 
\begin{align*}
\norm{w_j^{k} - w_j^{k-1}} 
\le  \frac{\eta}{n\sqrt{m}} \sum_{i=1}^n |y_i - \hat{y}_i(k-1)|
\le \frac{\eta}{\sqrt{nm}} \norm{ \bm{y} - \hat{\bm{y}}(k-1)}.
\end{align*}
Then, by \eqref{eq:sum2}, we have
\[
\max_{1\le k \le t+1} \norm{w_j^k -w_j^0}
\le \sum_{k=1}^{t+1}  \norm{w_j^{k} -w_j^{k-1} }
\le \frac{1}{\sqrt{m}}\pth{\frac{1}{c_0} + 2\eta T c_1}.
\]
Thus, by Lemma \ref{lm: concentration of initialization}, it holds that 
\begin{equation}
\label{eq:bound-F}
\norm{\bm{K} - \bm{H}} + \sqrt{\frac{4}{m^2n} \sum_{i=1}^n \pth{\sum_{j=1}^m \indc{\abth{\iprod{w_j^0}{x_i}} \le R}}^2} 
\le  \frac{4}{\sqrt{2\pi}\sqrt{m}}\pth{\frac{1}{c_0} + 2\eta T c_1} + 4\sqrt{ \frac{\log (4n/\delta) }{m}}.
\end{equation}
Substituting 
\eqref{eq:bound-F} into \eqref{eq:induction-next}, we finish the induction.

\section{Proof of Theorem \ref{thm: main theorem}}
\label{app: thm: main theorem}
We prove  \eqref{eq: sufficiency}  through exploring the structure of $f^*$ and using the concentration of spectral projection. In a sense, $\frac{1}{\sqrt{n}}\pth{\bm{I} - \eta\bm{K}}^{t} \bm{y}$ approximates $\pth{\calI - \eta L_{\calK}}^t f^*$ w.\,r.\,t.\,some properly chosen norm. Here $\calI$ is the identity operator, i.e., $\calI f = f$ for each $f\in L^2(\calS^{d-1}, \rho)$. 
Theorem \ref{thm: main theorem} follows immediately from Theorem \ref{thm: convergence rate under sufficient condition} and the following lemma. 
\begin{lemma}
\label{lm: empirical projection}
For any $\ell \ge 1$ such that $\mu_i>0, ~ \text{for } i=1, \cdots, \ell$, let 
\[ 
\epsilon(f^*, \ell) : = \sup_{x\in \calS^{d-1}}\abth{f^* (x)- (\sum_{1\le i\le \ell}P_{\mu_i} f^*)(x)}.
\]  
Then given $\delta \in (0, \frac{1}{4})$, if $n> \frac{256\log \frac{2}{\delta}}{(\lambda_{m_{\ell}} -  \lambda_{m_{\ell}+1})^2}$, 
with probability at least $1-2\delta$ it holds that 
\begin{align*}
\norm{\frac{1}{\sqrt{n}}\pth{\bm{I} - \eta\bm{K}}^{t}\bm{y}} 
&\le   \pth{1-\frac{3}{4}\eta\lambda_{m_{\ell}}}^{t}  + \frac{8\sqrt{2}  \sqrt{\log \frac{2}{\delta}} }{(\lambda_{m_\ell} -  \lambda_{{m_\ell}+1}) \sqrt{n}}+ \sqrt{2}\epsilon(\ell, f^*).   
\end{align*}
\end{lemma}

\begin{proof}
Since $\bm{K}$ is symmetric, 
we have $\bm{K}=\sum_{i=1}^n \hat\lambda_i \hat u_i \hat u_i^\top$, where $\hat\lambda_i$s 
are in an non-increasing order, $0\le \hat\lambda_i\le 1$, and $\norm{\hat u_i}=1$. 
For each $i$, define a function $\hat{\phi}_i$ over $\sth{x_k:  k \in [n] }$ by $\hat{\phi}_i(x_k)=\sqrt{n}\hat u_i(k)$ for $k\in [n]$.
Let $\rho(n)$ be the empirical distribution of $\sth{x_k: k \in [n] }$. Define $\iprod{\cdot}{\cdot}_{\rho(n)}$ as 
\begin{align}
\label{eq: empirical inner product}
\iprod{f}{g}_{\rho(n)}: = \frac{1}{n} \sum_{k=1}^n f(x_k) g(x_k). 
\end{align}
Notably, $\iprod{\cdot}{\cdot}_{\rho(n)}$ is similar to that of $\iprod{\cdot}{\cdot}_{\rho}$ but with a different measure.  
By definition $\sth{\hat{\phi}_i: 1\le i \le n}$ is a set of $n$ orthonormal functions w.r.t.\,the inner product $\iprod{\cdot}{\cdot}_{\rho(n)}$. 
It holds that 
\[
\frac{1}{\sqrt{n}}\pth{\bm{I} - \eta\bm{K}}^{t} \bm{y}
=\frac{1}{\sqrt{n}}\sum_{i=1}^n (1-\eta\hat\lambda_i)^t(\hat u_i^\top \bm{y})\hat u_i
= \sum_{i=1}^n (1-\eta \hat{\lambda}_{i})^t \Iprod{\hat{\phi}_i}{f^*}_{\rho(n)} \hat{u}_i.  
\]
Henceforth, we assume that $m_{\ell} < n$; the case $m_{\ell} \ge n$ can be shown similarly with the fact that $\lambda_{m_{\ell}} \le \lambda_n$. 
Since $\Iprod{\hat{\phi}_i}{f^*}_{\rho(n)}=\frac{1}{n}\sum_{i=1}^n (\hat u_i^\top \bm{y})^2=\frac{1}{n}\norm{\bm{y}}^2\le 1$, we have 
\begin{equation}
\label{eq:norm-main}
\norm{\frac{1}{\sqrt{n}}(\bm{I} - \eta\bm{K})^{t} \bm{y}}^2
\le (1-\eta \hat{\lambda}_{m_\ell})^{2t}+\sum_{i=m_\ell +1}^n \Iprod{\hat{\phi}_i}{f^*}_{\rho(n)}^2.
\end{equation}

Next we analyze the second term in \eqref{eq:norm-main}.
Let $\phi_1,\phi_2,\dots$ be orthonormal eigenfunctions of $L_{\calK}$ with strictly positive eigenvalues 
$\lambda_1, \lambda_2, \cdots $, respectively. 
%
Let $\gamma_j : = \iprod{f^*}{\phi_j}_{\rho}$. 
It holds that 
\begin{align}
\label{eq: function decomposition bound}
\nonumber
&\sum_{i=m_\ell +1}^n \pth{1-\eta \hat{\lambda}_{i}}^{2t} \pth{\iprod{\hat{\phi}_i}{f^*}_{\rho(n)}}^2
\overset{(a)}{\le} \sum_{i=m_\ell +1}^n \pth{\iprod{\hat{\phi}_i}{f^*}_{\rho(n)}}^2\\
 \nonumber
& = \sum_{i=m_\ell +1}^n \pth{\iprod{\hat{\phi}_i}{\sum_{j=1}^{m_{\ell}} \gamma_j \phi_j}_{\rho(n)} +   \iprod{\hat{\phi}_i}{f^* - \sum_{j=1}^{m_{\ell}} \gamma_j \phi_j}_{\rho(n)}}^2 \\
& \le 2\sum_{i=m_\ell +1}^n \pth{\iprod{\hat{\phi}_i}{\sum_{j=1}^{m_{\ell}} \gamma_j \phi_j}_{\rho(n)} }^2 + 2  \sum_{i=m_\ell +1}^n \pth{ \iprod{\hat{\phi}_i}{f^* - \sum_{j=1}^{m_{\ell}} \gamma_j \phi_j}_{\rho(n)}}^2,
\end{align}
where inequality (a) holds because that $0<\hat{\lambda}_i \le 1$. 
The first term in 
\eqref{eq: function decomposition bound} can be bounded as 
\begin{align}
\label{eq: thm 4 aaa}
\sum_{i=m_{\ell}+1}^n \iprod{\hat{\phi}_i}{\sum_{j=1}^{m_{\ell}} \gamma_j \phi_j}^2_{\rho(n)}  
& \overset{(a)}{\le}   \sum_{i=m_{\ell}+1}^n \pth{\sum_{j=1}^{m_{\ell}} \gamma_j^2} \sum_{j=1}^{m_{\ell}} \iprod{\hat{\phi}_i}{\phi_j}^2_{\rho(n)}
\overset{(b)}{\le} \sum_{i=m_{\ell}+1}^n \sum_{j=1}^{m_{\ell}} \iprod{\hat{\phi}_i}{\phi_j}^2_{\rho(n)}, 
\end{align}
where inequality (a) follows from Cauchy-Schwarz inequality, and inequality (b) is true because that $\sum_{j=1}^{m_{\ell}} \gamma_j^2\le 1$. In addition, for any $\delta\in (0,1)$, with probability at least $1-\frac{\delta}{2}$, it holds that 
\begin{align}
\label{eq: projection bound}
\sum_{i=m_{\ell}+1}^n \sum_{j=1}^{m_{\ell}} \iprod{\hat{\phi}_i}{\phi_j}^2_{\rho(n)} 
\le \frac{64\hat{\lambda}_{m_{\ell}+1}\log \frac{2}{\delta}}{\lambda_{m_{\ell}}(\lambda_{m_{\ell}} - \lambda_{m_{\ell}+1})^2 n}. 
\end{align}
We postpone the proof of \eqref{eq: projection bound} to Section \ref{subsec: eq 44}. 
%
We bound the second term in 
 \eqref{eq: function decomposition bound} as  
\begin{align}
\label{eq: thm 4 bbb}
\sum_{i=m_{\ell}+1}^n \iprod{\hat{\phi}_i}{f^* - \sum_{j=1}^{m_{\ell}} \gamma_j \phi_j}^2_{\rho(n)}  
\le \frac{1}{n} \sum_{k=1}^n \pth{f^*(x_k) - \sum_{j=1}^{m_{\ell}} \gamma_j \phi_j (x_k)}^2 
\le \epsilon^2(f^*, \ell). 
\end{align}
In addition, by  \eqref{eq:eigen-concentration-n} and the assumption that 
$n> \frac{256\log \frac{2}{\delta}}{(\lambda_{m_{\ell}} -  \lambda_{m_{\ell}+1})^2}$, 
with probability at least $1-\delta$, 
\begin{align}
\label{eq: thm: eigen bound}
\hat{\lambda}_{m_\ell}\ge \frac{3} {4}\lambda_{m_\ell}, ~\text{and}~ \hat{\lambda}_{m_\ell+1} \le\lambda_{m_\ell}. 
\end{align}
By \eqref{eq: function decomposition bound}, \eqref{eq: thm 4 aaa}, \eqref{eq: projection bound}, \eqref{eq: thm 4 bbb}, and \eqref{eq: thm: eigen bound}, we continue to bound \eqref{eq:norm-main} as: 
for any $\delta\in (0, \frac{1}{4})$, with probability at least $1-2\delta$, 
\begin{align*}
\norm{\frac{1}{\sqrt{n}}\pth{\bm{I} - \eta\bm{K}}^{t}\bm{y} }^2 
&\le   \pth{1-\frac{3\eta}{4}\lambda_{m_\ell} }^{2t}  +  \frac{ 128 \log 2/\delta}{(\lambda_{m_\ell} -  \lambda_{{m_\ell}+1})^2 n} + 2\epsilon^2(\ell, f^*).  
\end{align*}

\subsection{Proof of Eq.\ \eqref{eq: projection bound}}
\label{subsec: eq 44}
\paragraph{Preliminaries}
Recall from \eqref{eq:eigen-concentration-n} that the spectrum of $\bm{K}$ concentrates on the spectrum of the integral operator $L_{\calK}$. 
To show \eqref{eq: projection bound}, we need to know how $\phi_i, \,i\ge 1$ the eigenfunctions of $L_{\calK}$ and $\hat{\phi}_i,\, 1\le i \le n$ the eigenfunctions of $\bm{K}$ are related. Though both $L_{\calK}$ and $\bm{K}$ are defined w.\,r.\,t.\,the kernel function $\calK$ (defined in \eqref{eq: initial kernel random}), investigating this relation is not easy. This is because that $\phi_i$ is defined on $L^2(\calS^{d-1}, \rho)$, whereas $\hat{\phi}_i$ is defined on $L^2(\calS^{d-1}, \rho(n))$. 
To overcome this difficulty, we relate $L_{\calK}$ and $\bm{K}$ to two linear operators $T_{\calH}$ and $T_n$, respectively, on $\calH$ the {\em reproducing kernel Hilbert space} (RKHS) associated with the kernel function $\calK$.  
In particular, we define $T_{\calH}$ and $T_n$ by  
\begin{align*}
&T_{\calH}f=\int_{\calS^{d-1}} \Iprod{f}{\calK_x}_\calH \calK_x \diff \rho(x), ~ \text{and} ~~   T_n f =\frac{1}{n}\sum_{i=1}^n\Iprod{\cdot}{\calK_{x_i}}_\calH \calK_{x_i}.
\end{align*}
Here $\Iprod{\cdot}{\cdot}_\calH$ is the inner product with the RKHS $\calH$ that satisfies $f(x)=\Iprod{f}{\calK_x}_\calH$ for $f\in \calH$, where $\calK_x=\calK(x,\cdot)$. It has been shown  that the spectra of $L_{\calK}$ and $T_{\calH}$ are the same, possibly up to the zero, and that the spectra of $\bm{K}$ and $T_n$ are the same, possibly up to the zero. More importantly, clear correspondences between $L_{\calK}$ and $T_{\calH}$ and between $\bm{K}$ and $T_n$ are established. See \cite[item 2 of Proposition 8]{rosasco2010learning} and \cite[item 2 of Proposition 9]{rosasco2010learning}  for details. Notably, there is a notational issue in \cite{rosasco2010learning} which leads to an error in the
multipliers in the correspondences. But this error can be fixed easily, and our calculation reflects this correction. 

\paragraph{Proof}
We first show that $\sum_{i=m_{\ell}+1}^n \sum_{j=1}^{m_{\ell}} \iprod{\hat{\phi}_i}{\phi_j}^2_{\rho(n)}$ can be upper bounded with 
(1) the difference between the projection of $T_{\calH}$ onto its first $m_{\ell}$ eigenfunctions and that of $T_{n}$, 
and (2) the correspondences between the eigenfunctions of $L_{\calK}$ and $T_{\calH}$ and between that of $\bm{K}$ and $T_n$. 
Then we apply existing bound on the projection difference to conclude the proof. 

Let $v_1, \cdots, v_{m_{\ell}}, \cdots $ be the orthonormal set of functions in $\calH$ that related to $\phi_1, \cdots, \phi_{m_{\ell}}, \cdots$ by the relation given by \cite[item 2 of Proposition 8]{rosasco2010learning}. Similarly, let $\hat{v}_1, \cdots, \hat{v}_n$ be the corresponding Nystrom extension given by \cite[item 2 of Proposition 9]{rosasco2010learning}. Complete $\sth{v_i}_{i\ge1}$ and $\sth{\hat{v}_i}_{1\le i \le n}$, respectively, to orthonormal bases of $\calH$. 
%
Define two projection operators as follows: 
\[
P^{T_{\calH}}=\sum_{j=1}^{{m_\ell}} \Iprod{\cdot}{v_j}_\calH v_j,\quad P^{T_n}=\sum_{j=1}^{{m_\ell}} \Iprod{\cdot}{\hat v_j}_\calH \hat v_j.
\]
Since both $(v_j)_{j\ge1}$ and $(\hat{v}_j)_{j\ge1}$ are orthonormal bases for $\calH$, it is true that 
\[
\|P^{T_n} - P^{T_{\calH}}\|^2_{HS}  = \sum_{i\ge 1, j\ge 1}\abth{\iprod{\pth{P^{T_n} - P^{T_{\calH}}}\hat{v}_i}{v_j}}^2, 
\]
where $\norm{\cdot}_{HS}$ denotes the Hilbert–Schmidt norm defined as $\norm{A}_{HS}^2=\sum_{i\in I}\norm{Ae_i}^2$ for an orthonormal basis $\{e_i:i\in I\}$. 
By definition of $P^{T_{\calH}}$ and $P^{T_{n}}$, we have 
\begin{align*}
\iprod{\pth{P^{T_n} - P^{T_{\calH}}}\hat{v}_i}{v_j} &= \iprod{P^{T_n}\hat{v}_i}{v_j} - \iprod{P^{T_{\calH}}\hat{v}_i}{v_j}
& = 
\begin{cases}
0, ~~ & \text{if} ~ 1\le i\le {m_\ell}, \& ~ 1\le j \le {m_\ell}; \\
\iprod{\hat{v}_i}{v_j}_{\calH},  ~~ & \text{if} ~ 1\le i\le {m_\ell}, \& ~  j \ge {m_\ell}+1;\\
- \iprod{\hat{v}_i}{v_j}_{\calH},  ~~ & \text{if} ~ i\ge {m_\ell}+1, \& ~  1\le j \le {m_\ell}; \\
0,  ~~ & \text{if} ~ i\ge {m_\ell}+1, \& ~  j \ge {m_\ell}+1. 
\end{cases}
\end{align*}
Thus we get 
\begin{align*}
\|P^{T_n} - P^{T_{\calH}}\|^2_{HS}  &= \sum_{i=1}^{m_\ell} \sum_{j\ge {m_\ell}+1} \pth{\iprod{\hat{v}_i}{v_j}_{\calH}}^2 + \sum_{i\ge {m_\ell}+1} \sum_{j=1}^{m_\ell} \pth{\iprod{\hat{v}_i}{v_j}_{\calH}}^2 
\ge \sum_{i = {m_\ell}+1}^n  \sum_{j=1}^{m_\ell} \pth{\iprod{\hat{v}_i}{v_j}_{\calH}}^2. 
\end{align*}
Since with probability 1 over the data generation $\hat{\lambda}_i>0$ for $i=1, \cdots, n$,  
for $1\le i \le n$, we have 
\[
\pth{\iprod{\hat{v}_i}{v_j}_{\calH}}^2 = \frac{1}{\hat{\lambda}_i} \iprod{\hat{\phi}_i}{v_j}^2_{\rho(n)}. 
\]
So it holds that 
%
\begin{align*}
\sum_{i= {m_\ell}+1}^n \sum_{j=1}^{m_\ell} \pth{\iprod{\hat{v}_i}{v_j}_{\calH}}^2 
& \ge \frac{1}{\hat{\lambda}_{{m_\ell}+1}} \sum_{i= {m_\ell}+1}^n \sum_{j=1}^{m_\ell} \pth{\iprod{\hat{\phi}_i}{v_j}_{\rho(n)}}^2  
 \overset{(b)}{=}  \frac{1}{\hat{\lambda}_{{m_\ell}+1}} \sum_{i= {m_\ell}+1}^n \sum_{j=1}^{m_\ell} \pth{\iprod{\hat{\phi}_i}{\sqrt{\lambda_j}\phi_j}_{\rho(n)}}^2\\
& \ge  \frac{\lambda_{m_{\ell}}}{\hat{\lambda}_{{m_\ell}+1}} \sum_{i= {m_\ell}+1}^n \sum_{j=1}^{m_\ell} \pth{\iprod{\hat{\phi}_i}{\phi_j}_{\rho(n)}}^2, 
\end{align*}
where equality (b) follows from \cite[Proposition 8, item 2]{rosasco2010learning}. 
Since $n>\frac{256\log \frac{2}{\delta}}{(\lambda_{m_{\ell}} -  \lambda_{m_{\ell}+1})^2}$, 
by \cite[Theorem 7 and Proposition 6]{rosasco2010learning}, it holds that with probability at least $1-\frac{\delta}{2}$,
\[
\|P^{T_n} - P^{T_{\calH}}\|^2_{HS} \le \frac{64 \log \frac{2}{\delta}}{(\lambda_{m_\ell} -  \lambda_{{m_\ell}+1})^2 n},  
\]
finishing the proof of Eq.\ \eqref{eq: projection bound}. 


%
\end{proof}

\newcommand{\Floor}[1]{\lfloor {#1} \rfloor}

\section{Harmonic analysis on spheres}
\label{app: Harmonic analysis: gegenbauer}
Throughout this section, we consider uniform distribution $\rho$ on the unit sphere in $\reals^d$ with $d\ge 3$, and we consider functions on 
on $\calS^{d-1}$. For ease of exposition,  we do not explicitly write out the dependence on $d$ in the notations.

%
Let $\calH_\ell$ denote the space of degree-$\ell$ homogeneous harmonic polynomials on $\calS^{d-1}$:
\[
\calH_\ell=\sth{P: \calS^{d-1}\mapsto \reals: P(x)=\sum_{|\alpha|=\ell}c_\alpha x^{\alpha},\Delta P=0},
\]
where $x^{\alpha}=x_1^{\alpha_1}\cdots x_d^{\alpha_d}$ is a monomial with degree $|\alpha|=\alpha_1 + \cdots + \alpha_d$, $c_{\alpha}\in \reals$, and $\Delta$ is the Laplacian operator. 
The dimension of $\calH_\ell$ is denoted by $N_\ell=\frac{(2\ell+d-2)(\ell+d-3)!}{\ell!(d-2)!}$. For any $\ell$ and $\ell^{\prime}$, the spaces $\calH_\ell$ and $\calH_{\ell^{\prime}}$ are orthogonal to each other. 

%
%
The Gegenbauer polynomials, denoted by $C_{\ell}^{(\lambda)}$ for $\lambda>-\frac{1}{2}$ and $\ell =0, 1, \cdots $, are defined on $[-1, 1]$
as 
\begin{align}
\label{eq: gegenbauer polynomials}
C_\ell^{(\lambda)}(u)=\sum_{k=0}^{\Floor{\ell/2}}(-1)^k\frac{\Gamma(\ell-k+\lambda)}{\Gamma(\lambda)k!(\ell-2k)!}(2u)^{\ell-2k},
\end{align}
where $\Gamma(v) := \int_{0}^{\infty} z^{v-1}e^{-z} \diff z$. Notably, $\Gamma(v+1) = z \Gamma(v)$. 
The cases $\lambda=0,\frac{1}{2},1$ correspond to Chebyshev polynomials of the first kind, Legendre polynomials, Chebyshev polynomials of the second kind, respectively. 
It has been shown that \cite[Section 4.1(2), Section 4.7]{orthogonal.poly} 
for $\lambda \not=0$ Gegenbauer polynomials are orthogonal with the weight function $w_\lambda(u)=(1-u^2)^{\lambda-\frac{1}{2}}$:
\begin{align*}
\int_{-1}^1C_\ell^{(\lambda)}(u)C_k^{(\lambda)}(u) w_\lambda(u)\diff u =  \frac{\pi 2^{1-2\lambda}\Gamma(\ell+2\lambda)}{\ell!(\ell+\lambda)(\Gamma(\lambda))^2}\delta_{k,\ell},
\end{align*}
where $\delta_{k,\ell} =1$ if $k=\ell$ and $\delta_{k,\ell}=0$ otherwise. 
The orthogonality can be equivalently written as 
\[
\int_{0}^\pi C_\ell^{(\lambda)}(\cos\theta)C_k^{(\lambda)}(\cos\theta) \sin^{2\lambda}\theta\diff \theta =  \frac{\pi 2^{1-2\lambda}\Gamma(\ell+2\lambda)}{\ell!(\ell+\lambda)(\Gamma(\lambda))^2}\delta_{k,\ell}.
\]

For each $\ell \in \naturals$, there exists a set of orthonormal basis $\{Y_{\ell,i}:i=1,\dots,N_{\ell}\}$ for $\calH_{\ell}$ w.\,r.\,t.\,the uniform distribution $\rho$ that can be written in terms of $C_\ell^{(\lambda)}$ in \cite[Theorem 1.5.1]{DX2013} as 
\begin{align}
\label{eq: addition theorem}
C_\ell^{(\lambda)}(\Iprod{x}{y}) = \frac{\lambda}{\ell+\lambda}\sum_{i=1}^{N_{\ell}}Y_{\ell,i}(x)Y_{\ell,i}(y),\quad \lambda=\frac{d-2}{2}. 
\end{align}
This is known as the {\em addition theorem}. 
Therefore, a function of the form $f(x,y)=f(\Iprod{x}{y})$ (i.e., the value of $f(x,y)$ depends on $x$ and $y$ through their angle $\Iprod{x}{y}$ only) 
can be expanded under $C_\ell^{(\lambda)}$ as 
\begin{align}
f(x,y)&=f(\Iprod{x}{y}) =\sum_{\ell\ge 0}\alpha_\ell C_\ell^{(\lambda)}(u) 
 = \sum_{\ell\ge 0} \frac{\alpha_\ell\lambda}{\ell+\lambda} \sum_{i=1}^{N_\ell}Y_{\ell,i}(x)Y_{\ell,i}(y) ~~ \label{eq: kernel ortho decom},
\end{align}
where  $u=\iprod{x}{y}$, $\lambda=\frac{d-2}{2}$,  and 
\[
\alpha_\ell=\frac{\int_{-1}^1f(u)C_\ell^{(\lambda)}(u)w_\lambda(u)\diff u}{\int_{-1}^1(C_\ell^{(\lambda)}(u))^2w_\lambda(u)\diff u}. 
\]

For the kernel function defined in \eqref{eq: kernel with bias},  
it can be expanded as  
\begin{align*}
 \calK(x,s)&=\sum_{\ell\ge 0}  \beta_{\ell} \sum_{i=1}^{N_\ell}Y_{\ell,i}(x)Y_{\ell,i}(s), ~~~ \text{where } \beta_{\ell}: =  \frac{\alpha_\ell \frac{d-2}{2}}{\ell + \frac{d-2}{2}}, 
\end{align*}
where for each $\ell\ge 0$, $\alpha_{\ell}$ is the coefficient of $\calK(x,s)$ in the expansion into Gegenbauer polynomials,  
$ \beta_{\ell}$ is the eigenvalue associated with the space of degree--$\ell$ homogeneous harmonic polynomials on $\calS^{d-1}$, denoted by $\calH^{\ell}$, and  $Y_{\ell, i} $ for  $ i =1, \cdots, N_{\ell}$ are an orthonormal basis of $\calH^{\ell}$. Thus, the corresponding integral operator can be decomposed as $L_{\calK} = \sum_{\ell\ge 0}  \beta_{\ell} P_{\ell}$.

\end{document}